%% file: main.tex
\documentclass[11pt]{article}

\usepackage[numbers]{natbib}

\usepackage{tabularray}
\usepackage{booktabs}       
\usepackage{amsfonts}       
\usepackage{amsmath}
\usepackage{graphicx}
\usepackage{amsthm}
\usepackage{amssymb}
\usepackage{multirow}
\usepackage{makecell}
\usepackage[vlined,linesnumbered,ruled,resetcount]{algorithm2e}
\usepackage[colorlinks,linkcolor=magenta,filecolor=blue,citecolor=blue,urlcolor=blue]{hyperref}%
\usepackage[top=1in, left=1in, right=1in, bottom=1in]{geometry}
\usepackage{subcaption}
\usepackage{amssymb}
\usepackage{pifont}
\usepackage{mathtools}
\usepackage{stmaryrd}
\usepackage[T1]{fontenc}
\usepackage{wrapfig,epsfig}
\usepackage{authblk}
\usepackage{enumitem}
\usepackage{algorithmic}
\usepackage{threeparttable}
\usepackage{xcolor}
\usepackage{colortbl}
\newcommand{\gr}{\rowcolor[gray]{.95}}
\newcommand{\wc}{\cellcolor{white}}

\newtheorem{theorem}{Theorem}[section]
\newtheorem{lemma}[theorem]{Lemma}

\theoremstyle{definition}

\renewcommand{\epsilon}{\varepsilon}
\renewcommand{\phi}{\varphi}

\makeatletter
\newcommand*{\RN}[1]{\expandafter\@slowromancap\romannumeral #1@}
\makeatother

\makeatletter
\newcommand{\printfnsymbol}[1]{%
  \textsuperscript{\@fnsymbol{#1}}%
}
\makeatother

\title{Determining Layer-wise Sparsity for Large Language Models Through a Theoretical Perspective}

\author[1]{Weizhong Huang}
\author[1]{Yuxin Zhang}
\author[1,2]{Xiawu Zheng}
\author[1]{Fei Chao}
\author[1,2]{Rongrong Ji\thanks{Corresponding Author: rrji@xmu.edu.cn.}}
\affil[1]{Key Laboratory of Multimedia Trusted Perception and Efficient Computing, Ministry of Education of China, Xiamen University, 361005, P.R. China.}
\affil[2]{Institute of Artificial Intelligence, Xiamen University.}
\affil[ ]{\texttt{\{huangwz,yuxinzhang\}@stu.xmu.edu.cn}, \texttt{\{zhengxiawu,fchao,rrji\}@xmu.edu.cn}}
\date{}

\begin{document}

\maketitle
\begin{abstract}
    \input{abstract}
\end{abstract}

\newpage
\clearpage
\input{intro}
\input{related}
\input{method}
\input{exp}
\input{conclusion}

\bibliographystyle{plainnat}
\bibliography{ref}

\newpage
\clearpage
\appendix
\begin{center}
{\LARGE \textbf{Appendix}}
\end{center}
\input{appendix}
\end{document}

%% file: abstract.tex
In this paper, we address the challenge of determining the layer-wise sparsity rates of large language models (LLMs) through a theoretical perspective. Specifically, we identify a critical issue of \textbf{``reconstruction error explosion''} in existing LLMs sparsification methods. This refers to the cumulative effect of reconstruction errors throughout the sparsification process, where errors from earlier layers propagate and amplify in subsequent layers. As a result, the overall reconstruction error increases significantly, leading to a substantial degradation in model performance.
Through theoretical analysis, we derive a simple yet effective approach to layer-wise sparsity allocation that mitigates this issue. Our method uses a monotonically increasing arithmetic progression, reducing the process of determining sparsity rates for multiple layers to the determination of a single common difference hyperparameter. Remarkably, this allows for the optimal layer-wise sparsity rates to be identified with just a few trials. Both our theoretical analysis and experimental results demonstrate that this sparsity allocation scheme is near optimal.
Extensive experiments show that our method significantly improves the performance of sparse LLMs across various architectures, outperforming existing layer-wise sparsity methods. Furthermore, it enhances the performance of various compression techniques and is applicable to vision and multimodal models. Notably, our method achieves a reduction of 52.10 in perplexity for the 70$\%$ sparse LLaMA2-7B model obtained via Wanda, improves average zero-shot accuracy by 10.50$\%$, and delivers speedups of 2.63$\times$ and 2.23$\times$ on CPU and GPU, respectively.

%% file: intro.tex
\section{Introduction}\label{sec:intro}
Large Language Models (LLMs) have demonstrated outstanding capabilities in various natural language processing tasks~\cite{llama3, yang2024qwen2, liu2024deepseek}. However, their vast number of parameters and high computational demands present significant challenges to model deployment, impeding further applications~\cite{zhu2024survey, wang2024model}. Network sparsity~\cite{rao2021dynamicvit, paul2022unmasking, wangntk} methods remove less important parameters from LLMs, enabling model compression without sacrificing performance. This can significantly reduce the model's memory footprint and computational complexity~\cite{li2024lorap, an2024fluctuation}. Existing sparsity methods for LLMs, such as SparseGPT~\cite{frantar2023sparsegpt} and Wanda~\cite{sun2023simple}, adopt a post-training approach which prune all weights in one-shot and can obtain sparse LLMs without the need for additional fine-tuning. 

However, these sparsity methods set a uniform layer-wise sparsity rate for different layers of LLMs, without considering the varying importance of each layer, which harms the accuracy of sparse LLMs. To address the above issue, many studies have proposed various methods to determine the layer-wise sparsity rate of LLMs. Based on their methodological designs, we categorize these approaches into two main groups:

\begin{figure*}[t]
    \centering
    \begin{minipage}[b]{0.3\textwidth}
        \centering
        \includegraphics[width=\textwidth]{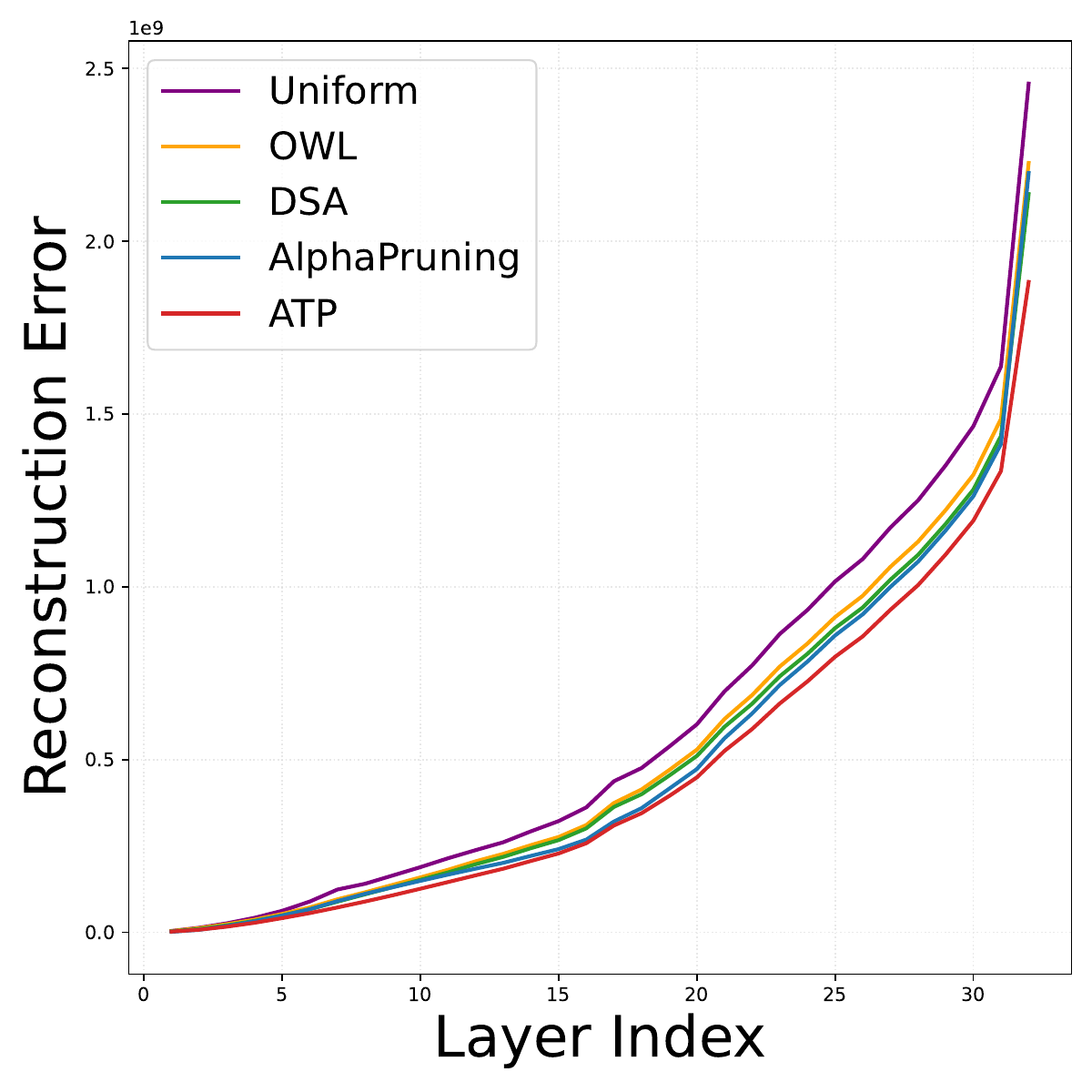}
    \end{minipage}
    \hfill
    \begin{minipage}[b]{0.69\textwidth}
        \centering
        \includegraphics[width=\textwidth]{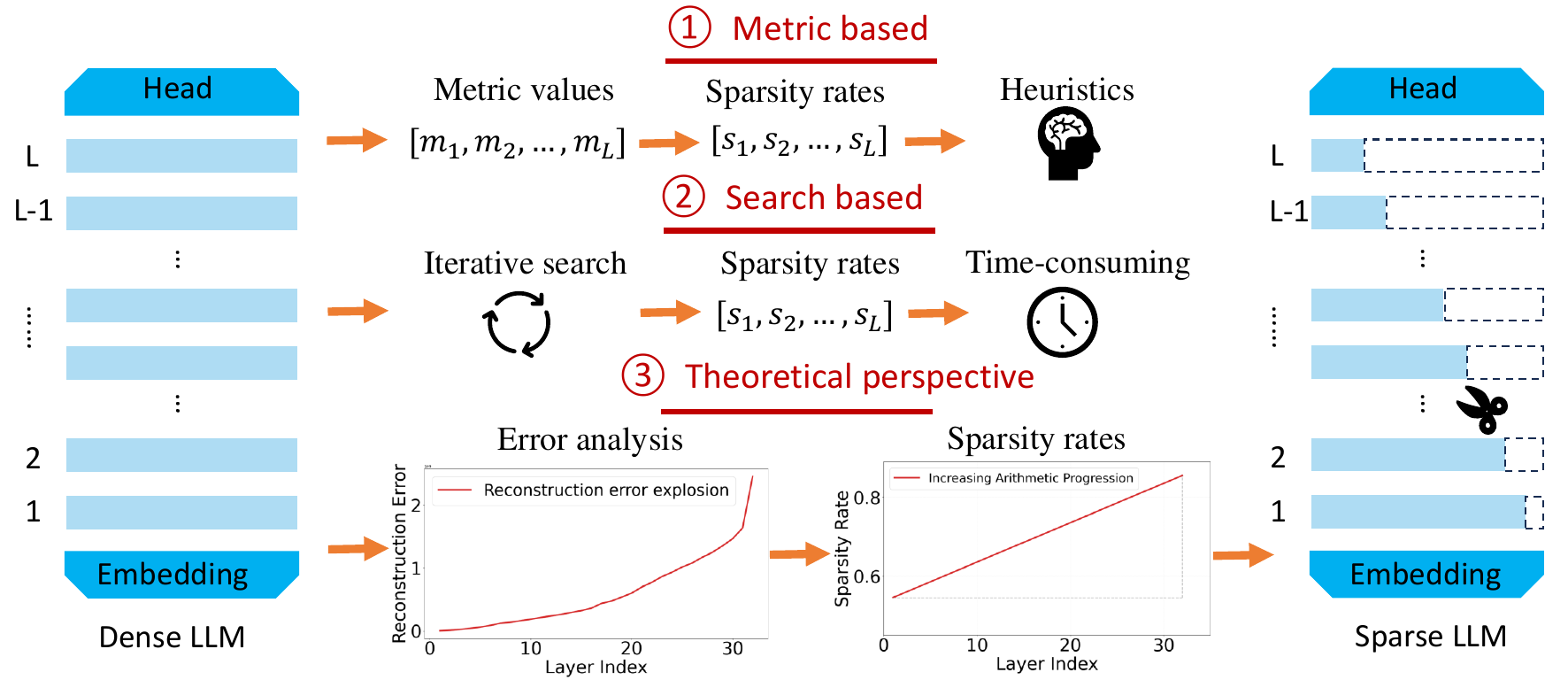}
    \end{minipage}
    \caption{(Left) shows the comparison of reconstruction error among different layer-wise sparsity methods. All methods face the problem of \textbf{``reconstruction error explosion''}; however, our method achieves lower reconstruction error compared to other methods. (Right) presents a comparison between our method and other layer-wise sparsity methods. The metric-based method calculates the importance of each layer to obtain the sparsity rate. However, this method is heuristically designed by human experts and is not optimal. And the search-based method requires a large number of iterative searches, which is time-consuming. In contrast, we analyze the causes of \textbf{``reconstruction error explosion''} from a theoretical perspective, and deduce theoretically that using a monotonically increasing arithmetic progression to determine the layer-wise sparsity rate can alleviate the problem of \textbf{``reconstruction error explosion''}.}
    \label{fig:framework}
\end{figure*}

\textbf{Metric based methods.} These methods determine the importance of each layer of LLMs through hand-crafted metrics, thereby obtaining the sparsity rate of each layer. For example, OWL~\cite{yin2023outlier} proposes an outlier weighted layer importance metric. By setting the sparsity rate to be inversely proportional to the outlier ratio, it effectively protects the layers with a higher ratio of outliers. AlphaPruning~\cite{lu2024alphapruning} utilizes the heavy-tailed self- regularization theory~\cite{martin2019traditional}, especially the shape of the empirical spectral density~\cite{martin2021predicting} of the weight matrix, to determine the importance of each layer of LLMs. ALS~\cite{li2024adaptive} proposes an importance metric based on mutual information~\cite{tschannen2019mutual}, and sets a higher sparsity rate for layers with higher mutual information. Although these metrics have proven their effectiveness experimentally, however manually designing metric requires extensive validation and complex calculations are needed to obtain the sparsity rate of each layer. And most importantly, most of these methods lack theoretical analysis, making it impossible to ensure that the solutions obtained are optimal.

\textbf{Search based methods.} In addition to these heuristics designed by humans, recently, there have also been some methods that adopt a search-based approach to determine the layer-wise sparsity rate of LLMs. For example, DSA~\cite{li2024discovering} develops an expression discovery framework to explore potential sparsity rate allocation strategies and obtains layer-wise sparsity allocation function by searching. However, the evolutionary search method employed by DSA requires 2000 iterations. For large-scale LLMs with a vast number of parameters, this demands a search process lasting several days, which incurs a significant cost. In addition, in order to obtain the final allocation function, DSA designs a complex search space and process, which means the effectiveness of the method heavily depends on the experience of human experts.

In this paper, we rethink the approach of determining the layer-wise sparsity rate of LLMs, and derive the layer-wise sparsity rate of LLMs from reconstruction error perspective. Specifically, we first prove that increasing the sparsity rate leads to an increase in the reconstruction error of the corresponding layer. Additionally, we show that an increase in the reconstruction error of one layer causes an increase in the reconstruction error of subsequent layers. This implies that increasing the sparsity rate of earlier layers not only increases the reconstruction error of the corresponding layer, but also leads to an increase in the reconstruction errors of all subsequent layers. As the network propagates forward, the reconstruction errors accumulate, causing the total reconstruction error to grow significantly, thus causing \textbf{``reconstruction error explosion''} (See the left in Figure \ref{fig:framework}).

Through the above theoretical analysis, we provide a simple yet effective rule for determining the layer-wise sparsity rates of LLMs: \textit{the sparsity rate should be lower in earlier layers, and the layer-wise sparsity rates should follow a increasing pattern}. This approach effectively alleviates the issue of \textbf{``reconstruction error explosion''}, resulting in a well-performing sparse LLM. To achieve this, we use a monotonically increasing arithmetic progression to determine the sparsity rates for all layers of LLMs and employ grid search to find the common difference of the arithmetic progression. Since the range of valid values for the common difference is narrow, our search is highly efficient, and after only a few attempts, we can determine the common difference that yields the best accuracy.

Furthermore, we prove that the total reconstruction error obtained from the monotonically increasing sparsity scheme is less than that of any non-monotonically increasing sparsity scheme. This indicates that our method is theoretically close to the optimal solution. Additionally, we compare our sparsity rate scheme with the optimal solution obtained through Bayesian search~\cite{wu2019hyperparameter}, we find that our scheme is close to the optimal solution found by search. This indicates that our method is empirically close to the optimal solution.

To evaluate the effectiveness of our ATP \footnote{Determining layer-wise sparsity for LLMs through \textbf{A} \textbf{T}heoretical \textbf{P}erspective (\textbf{ATP}).} method, we conduct extensive experiments on LLMs of various architectures, with parameter counts ranging from 6.7 billion to 70 billion. Our evaluation metrics include perplexity, average accuracy across seven zero-shot datasets, and performance on arithmetic and knowledge reasoning tasks. Our ATP method demonstrate substantial improvements over existing post-training sparsity techniques, significantly surpassing other layer-wise sparsity methods. Notably, ATP reduce the perplexity of the 70$\%$ sparse LLaMA2-7B pruned using Wanda~\cite{sun2023simple} by 52.10 and increase average zero-shot accuracy by 10.50$\%$, outperforming the state-of-the-art AlphaPruning method~\cite{lu2024alphapruning} by 6.71 and 2.46$\%$, respectively. Additionally, ATP achieved 2.63$\times$ and 2.23$\times$ speedups on CPU and GPU, respectively, and require only $18$ minutes to compute layer-wise sparsity rates. Furthermore, we evaluate ATP's enhancements on various compression techniques, including N:M sparsity, structured pruning, and network quantization, as well as its benefits for sparse multimodal and sparse vision models. These experimental results clearly demonstrate that ATP provides substantial performance improvements for compressed models.

%% file: related.tex
\section{Related Work}\label{RelatedWork}
\paragraph{LLMs Sparsity.} 
Before the advent of LLMs, a variety of sparsity techniques had been developed for compressing models such as ResNet \citep{yu2022combinatorial, zhang2024how} , BERT \citep{xia2022structured, li2023losparse}, and ViT \citep{yu2022width, he2024pruning}. Meanwhile, researchers have developed several post-training sparsity methods specifically for LLMs. For example, SparseGPT \citep{frantar2023sparsegpt} uses the inverse of the Hessian matrix for pruning and pruned weight updates. Wanda \citep{sun2023simple} uses a metric that combines weight magnitude and input activation to prune LLMs, while Pruner-zero \citep{dong2024pruner} searches for symbolic pruning metric using genetic programming. Additionally, ALPS \citep{meng2024alps} uses an Alternating Direction Method of Multiplier (ADMM) \citep{boyd2011distributed}-based approach to prune LLMs in one-shot. The above methods focus on determining the mask within the layer of LLMs and setting a uniform layer-wise sparsity rate. Our work study the layer-wise sparsity allocation problem in sparse LLMs from the perspective of reconstruction error, thereby effectively improving the accuracy of above methods.
\paragraph{Layer-wise Sparsity.} Layer-wise sparsity rate determines the number of weights to be retained in each layer of the network \citep{lee2020layer, liu2022unreasonable}. To determine the layer-wise sparsity rate in LLMs, OWL\citep{yin2023outlier} proposes an outlier-weighted metric, Alphapruning \citep{lu2024alphapruning} uses heavy-tailed self-regularization theory \citep{martin2019traditional}, ALS \citep{li2024adaptive} proposes a layer redundancy metric based on mutual information \citep{kraskov2004estimating}, DSA \citep{li2024discovering} develops an expression discovery algorithm to explore potential sparsity allocation. However, the above metric and search based methods all lack theoretical proof of effectiveness and require complex calculations or search to obtain layer-wise sparsity rate. In comparison, our method directly derives the layer-wise sparsity rates based on a monotonically increasing arithmetic progression from the perspective of reconstruction error, requiring only the determination of the common difference to quickly obtain the sparsity rate for each layer. More importantly, we have validated the effectiveness of the above method through rigorous proof.

\paragraph{Reconstruction Error.} Reconstruction error is a metric for measuring the difference between the output of a compressed network and that of the original network. A smaller reconstruction error generally implies that the compressed network can better preserve the performance of the original network \citep{yun2021all, hubara2021accelerated, ma2023solving}. In order to minimize the reconstruction error of sparse LLMs, SparseGPT \citep{frantar2023sparsegpt} proposes a mask selection and weight update algorithm based on Hessian inverse and DSnoT \citep{zhang2023dynamic} proposes a training-free dynamic weight pruning and growing algorithm. In this paper, we explore the \textbf{``reconstruction error explosion''} problem in sparse LLMs. Specifically, the reconstruction error accumulates and magnifies across layers, leading to an extremely large overall reconstruction error, which undermines the model's accuracy. Therefore, we propose our layer-wise sparsity allocation method to alleviate the above \textbf{``reconstruction error explosion''} problem.

%% file: method.tex
\section{Methodology}
\paragraph{Notation.} In this paper, we use bold typeface indicates  matrices (\emph{e.g.,}$\boldsymbol{W, X}$) and calligraphic font represents loss functions or models (\emph{e.g.,}$\mathcal{L}, \mathcal{M}$). 

\subsection{Preliminaries}
Without loss of generality, we take each layer in the LLMs as the basic unit of analysis. These layers contain modules such as Attention \citep{vaswani2017attention}, MLP \citep{popescu2009multilayer}, LayerNorm \citep{lei2016layer}, and residual connections \citep{he2016deep}, etc. Consider an LLM composed of $L$ layers, we define the reconstruction error of the $i$-th layer ($i = 1, 2, \cdots, L$) as follows:
\begin{equation}
    \mathcal{L}(\boldsymbol{W}_i, \boldsymbol{X}_i) = \left\lVert \boldsymbol{W}_i\boldsymbol{X}_i - \widetilde{\boldsymbol{W}}_{i}\widetilde{\boldsymbol{X}}_{i}\right\rVert^2_F
\end{equation}
where \(\boldsymbol{W}_i \in \mathbb{R}^{c_{out}\times c_{in}}\) and \(\boldsymbol{X}_i \in \mathbb{R}^{c_{in}\times d}\) are the weights and input of the $i$-th layer respectively. \(\widetilde{\boldsymbol{W}}_{i}\) and \(\widetilde{\boldsymbol{X}}_{i}\) are the corresponding sparse versions, where \(c_{in}\) and \(c_{out}\) represent the number of input and output channels, \(d\) is the hidden dimension, and \(\left\lVert \cdot \right\rVert_F\) is Frobenius norm. 

Existing post-training sparsity methods all attempt to minimize the reconstruction error of sparse LLMs (\emph{e.g.,} SparseGPT and Wanda) and a large amount of experimental evidence in the these papers indicates that a sparse LLM with good accuracy often has a low reconstruction error. 

In the next section, we reveal that the existing post-training sparsity methods all have the problem of \textbf{``reconstruction error explosion''}.
\subsection{\textbf{``Reconstruction Error Explosion''
} in Sparse LLMs}\label{sec:error_analysis}
First, we analyze the relationship between the sparsity rate and the reconstruction error. Briefly, a higher sparsity rate leads to a higher reconstruction error. Formally, we propose Theorem \ref{theorem1}.
\begin{theorem}[\textbf{Effect of increased sparsity on reconstruction error}]
\label{theorem1}
Increasing the sparsity of the weights in the \(i\)-th layer will lead to an increase in the reconstruction error of this layer.
\end{theorem}
\begin{proof}[Proof of Theorem \ref{theorem1}]
We only consider the effect of sparse weights on the reconstruction error, and we ignore the error caused by the input at this time. Therefore, we consider the impact of the sparse weights of the \(i\)-th layer on the reconstruction error when the input is the same. That is, \(\widetilde{\boldsymbol{X}}_{i}=\boldsymbol{X}_{i}\). Then, the reconstruction error of the \(i\)-th layer is expressed as:
\begin{equation}
    \mathcal{L} = \left\lVert \boldsymbol{W}_i\boldsymbol{X}_i - \widetilde{\boldsymbol{W}}_{i}\boldsymbol{X}_{i}\right\rVert^2_F
\end{equation}
Consider two weights \(\widetilde{\boldsymbol{W}}_{i}^{(1)}\) and \(\widetilde{\boldsymbol{W}}_{i}^{(2)}\) of different sparsity, such that \(\widetilde{\boldsymbol{W}}_{i}^{(1)}\) has lower sparsity (i.e., has fewer zero elements). The difference in reconstruction error corresponding to these two sparse weights is:
\begin{equation}
\begin{aligned}
& \mathcal{L}^{(1)}-\mathcal{L}^{(2)} = \left\lVert \boldsymbol{W}_i\boldsymbol{X}_i - \widetilde{\boldsymbol{W}}_{i}^{(1)}\boldsymbol{X}_{i}\right\rVert^2_F\\
& - \left\lVert \boldsymbol{W}_i\boldsymbol{X}_i - \widetilde{\boldsymbol{W}}_{i}^{(2)}\boldsymbol{X}_{i}\right\rVert^2_F
= \left\lVert \boldsymbol{W}_i\boldsymbol{X}_i - \widetilde{\boldsymbol{W}}_{i}^{(1)}\boldsymbol{X}_{i}\right\rVert^2_F\\
& - \left\lVert (\boldsymbol{W}_i\boldsymbol{X}_i - \widetilde{\boldsymbol{W}}_{i}^{(1)}\boldsymbol{X}_{i})+(\widetilde{\boldsymbol{W}}_{i}^{(1)}\boldsymbol{X}_{i} - \widetilde{\boldsymbol{W}}_{i}^{(2)}\boldsymbol{X}_{i})\right\rVert^2_F \\
& = -2\langle(\boldsymbol{W}_i - \widetilde{\boldsymbol{W}}_{i}^{(1)})\boldsymbol{X}_{i}, (\widetilde{\boldsymbol{W}}_{i}^{(1)} - \widetilde{\boldsymbol{W}}_{i}^{(2)})\boldsymbol{X}_{i} \rangle_F \\
& - \left\lVert \widetilde{\boldsymbol{W}}_{i}^{(1)}\boldsymbol{X}_{i} - \widetilde{\boldsymbol{W}}_{i}^{(2)}\boldsymbol{X}_{i}\right\rVert^2_F\\
\end{aligned}
\end{equation}

The first term of the inner product, \((\boldsymbol{W}_i - \widetilde{\boldsymbol{W}}_{i}^{(1)})\boldsymbol{X}_{i}\), represents the error introduced by sparsifying the original weight matrix \(\boldsymbol{W}_i\) to obtain a less sparse version \(\widetilde{\boldsymbol{W}}_{i}^{(1)}\). The second term, \((\widetilde{\boldsymbol{W}}_{i}^{(1)} - \widetilde{\boldsymbol{W}}_{i}^{(2)})\boldsymbol{X}_{i}\), quantifies the additional error resulting from further sparsifying \(\widetilde{\boldsymbol{W}}_{i}^{(1)}\) to derive an even sparser matrix \(\widetilde{\boldsymbol{W}}_{i}^{(2)}\). Since both errors \((\boldsymbol{W}_i - \widetilde{\boldsymbol{W}}_{i}^{(1)})\boldsymbol{X}_{i}\) and \((\widetilde{\boldsymbol{W}}_{i}^{(1)} - \widetilde{\boldsymbol{W}}_{i}^{(2)})\boldsymbol{X}_{i}\) are generated through the same sparsification method, they point in similar directions within the vector space. This alignment ensures that their Frobenius inner product satisfies \(\langle(\boldsymbol{W}_i - \widetilde{\boldsymbol{W}}_{i}^{(1)})\boldsymbol{X}_{i}, (\widetilde{\boldsymbol{W}}_{i}^{(1)} - \widetilde{\boldsymbol{W}}_{i}^{(2)})\boldsymbol{X}_{i} \rangle_F > 0\). And since \(\left\lVert \widetilde{\boldsymbol{W}}_{i}^{(1)}\boldsymbol{X}_{i} - \widetilde{\boldsymbol{W}}_{i}^{(2)}\boldsymbol{X}_{i}\right\rVert^2_F>0 \), therefore:
\begin{equation}
\mathcal{L}^{(1)}-\mathcal{L}^{(2)} < 0
\end{equation}
Therefore, the reconstruction error $\mathcal{L}^{(1)}$ for lower sparsity is smaller than $\mathcal{L}^{(2)}$ for higher sparsity. In other words, increasing the sparsity will lead to an increase in the reconstruction error of this layer. So, we have completed the proof of the above theorem.
\end{proof}

On the other hand, we find that there is an cumulative effect of reconstruction error during the sparsification process. It is manifested as the reconstruction error in LLMs showing an increasing trend due to the influence of previous sparse layers. In other words, if the reconstruction error of the previous layer increases, the reconstruction error of the subsequent layers will also increase accordingly. Formally, we propose Theorem \ref{theorem2}. 
\begin{theorem}[\textbf{The cumulative effect of reconstruction error}]\label{theorem2}
When the reconstruction error of the \(i\)-th layer increases, it leads to an increase in the lower bound of the reconstruction error for the \((i+1)\)-th layer.
\end{theorem}

To prove Theorem \ref{theorem2}, we need to define the following lemma:
\begin{lemma}\label{lemma1} Let \(\boldsymbol{A} \in \mathbb{R}^{m \times n}\) and \(\boldsymbol{B} \in \mathbb{R}^{n \times p}\) be arbitrary matrices. Then, it holds that \( \left\lVert \boldsymbol{AB} \right\rVert_F^2 \geq \sigma_{\min}^2(\boldsymbol{A}) \left\lVert \boldsymbol{B} \right\rVert_F^2\), where \(\sigma_{\min}(\boldsymbol{A})\) denotes the smallest non-zero singular value of \(\boldsymbol{A}\).

\end{lemma}
The proof of the Lemma \ref{lemma1} can be found in Appendix \ref{sec:provelemma}. Now we formally prove Theorem \ref{theorem2}.

\begin{proof}[Proof of Theorem \ref{theorem2}]
The reconstruction error for the \((i+1)\)-th layer can be expressed as:
\begin{equation}
\begin{aligned}
&\mathcal{L}(\boldsymbol{W}_{i+1}, \boldsymbol{X}_{i+1}) 
= \left\lVert \boldsymbol{W}_{i+1}\boldsymbol{X}_{i+1} - \widetilde{\boldsymbol{W}}_{i+1}\widetilde{\boldsymbol{X}}_{i+1} \right\rVert_F^2 \\
&= \left\lVert (\boldsymbol{W}_{i+1} - \widetilde{\boldsymbol{W}}_{i+1})\boldsymbol{X}_{i+1} + \widetilde{\boldsymbol{W}}_{i+1}(\boldsymbol{X}_{i+1} - \widetilde{\boldsymbol{X}}_{i+1}) \right\rVert_F^2 \\
&= \left\lVert (\boldsymbol{W}_{i+1} - \widetilde{\boldsymbol{W}}_{i+1})\boldsymbol{X}_{i+1} \right\rVert_F^2 \\
&+ \left\lVert \widetilde{\boldsymbol{W}}_{i+1}(\boldsymbol{X}_{i+1} - \widetilde{\boldsymbol{X}}_{i+1}) \right\rVert_F^2 \\
& + 2\text{tr}(( \boldsymbol{W}_{i+1} - \widetilde{\boldsymbol{W}}_{i+1})\boldsymbol{X}_{i+1})^{\top}(\widetilde{\boldsymbol{W}}_{i+1}(\boldsymbol{X}_{i+1} - \widetilde{\boldsymbol{X}}_{i+1}))
\end{aligned}
\end{equation}
Since the third term is the trace of the product of two matrices, the first and second term is the square of the Frobenius norm, and considering that the weights and inputs in LLMs are matrices with large dimensions, the magnitude of the first and second term is much greater than the third term, so we ignore the third term and get:
\begin{equation}
\begin{aligned}
&\mathcal{L}(\boldsymbol{W}_{i+1}, \boldsymbol{X}_{i+1}) \approx \left\lVert (\boldsymbol{W}_{i+1} - \widetilde{\boldsymbol{W}}_{i+1})\boldsymbol{X}_{i+1} \right\rVert_F^2 \\
&+ \left\lVert \widetilde{\boldsymbol{W}}_{i+1}(\boldsymbol{X}_{i+1} - \widetilde{\boldsymbol{X}}_{i+1}) \right\rVert_F^2
\end{aligned}
\end{equation}
Since \(\left\lVert (\boldsymbol{W}_{i+1} - \widetilde{\boldsymbol{W}}_{i+1})\boldsymbol{X}_{i+1} \right\rVert_F^2 > 0 \). Therefore:
\begin{equation}
\mathcal{L}(\boldsymbol{W}_{i+1}, \boldsymbol{X}_{i+1}) >  \left\lVert \widetilde{\boldsymbol{W}}_{i+1}(\boldsymbol{X}_{i+1} - \widetilde{\boldsymbol{X}}_{i+1}) \right\rVert_F^2 
\end{equation}
According to Lemma \ref{lemma1}, we get:
\begin{equation}
\begin{aligned}
&\mathcal{L}(\boldsymbol{W}_{i+1}, \boldsymbol{X}_{i+1}) >  \sigma_{\min}^2(\widetilde{\boldsymbol{W}}_{i+1}) \left\lVert (\boldsymbol{X}_{i+1} - \widetilde{\boldsymbol{X}}_{i+1}) \right\rVert_F^2 \\
&= \sigma_{\min}^2(\widetilde{\boldsymbol{W}}_{i+1}) \left\lVert (\boldsymbol{W}_{i}\boldsymbol{X}_{i} - \widetilde{\boldsymbol{W}}_{i}\widetilde{\boldsymbol{X}}_{i}) \right\rVert_F^2 \\
&=\sigma_{\min}^2(\widetilde{\boldsymbol{W}}_{i+1})\mathcal{L}(\boldsymbol{W}_{i}, \boldsymbol{X}_{i})
\end{aligned}
\end{equation}
Since \(\sigma_{\min}^2(\widetilde{\boldsymbol{W}}_{i+1}) > 0\), we have proven that the increase of the reconstruction error of the \(i\)-th layer will lead to the increase of the lower bound of the reconstruction error of the \((i+1)\)-th layer.
\end{proof}

Theorem \ref{theorem2} shows that an increase in the reconstruction error of the previous layer in a sparse LLM usually leads to a further increase in the lower bound of the reconstruction error of the subsequent layer. In practice, this often means that an increase in the reconstruction error of the previous layer will lead to an increase in the reconstruction error of the subsequent layer. We have also observed this phenomenon in the left of Figure \ref{fig:framework}. We can see that when the reconstruction error of the earlier layers is smaller, the reconstruction error of the subsequent layers is also smaller. Conversely, when the reconstruction error of the earlier layers is larger, the reconstruction error of the subsequent layers is also larger.

According to Theorems \ref{theorem1} and \ref{theorem2}, we can easily get the following Theorem \ref{theorem3}:
\begin{theorem}[\textbf{Impact of the sparsity of the previous layer on the reconstruction error of the next layer.}]\label{theorem3}
Increasing the sparsity of the \(i\)-th layer will lead to an increase in the lower bound of the reconstruction error of the \((i+1)\)-th layer.
\end{theorem}
\begin{proof}[Proof of Theorem \ref{theorem3}]
According to Theorems \ref{theorem1} and \ref{theorem2}, we have the following relationship:
\begin{equation}
\begin{aligned}
&\text{sparsity of layer } i \uparrow \Longrightarrow \mathcal{L}(\boldsymbol{W}_i, \boldsymbol{X}_i) \uparrow \\
&\Longrightarrow \mathcal{L}(\boldsymbol{W}_{i+1}, \boldsymbol{X}_{i+1}) > \sigma_{\min}^2(\widetilde{\boldsymbol{W}}_{i+1}) \mathcal{L}(\boldsymbol{W}_i, \boldsymbol{X}_i) \uparrow.
\end{aligned}
\end{equation}
Therefore, we complete the proof of Theorem \ref{theorem3}.
\end{proof}

To summarize all the above, we can get the following logical chain of our paper: sparsity rate of \(1\)-st layer \(\propto\) reconstruction error of \(1\)-st layer \(\propto\) reconstruction error of \(L\)-th layer \(\propto\) total error \(\propto\) accuracy loss. That is, when the sparsity rate of the \(1\)-st layer increases, it will lead to an increase in the reconstruction error of this layer. As the reconstruction error accumulates continuously from the \(1\)-st layer to the \(L\)-th layer, the reconstruction error of the \(L\)-th layer also increases accordingly. Then, the reconstruction errors of all layers of the model will exhibit an explosion phenomenon, resulting in a serious decline in the accuracy of the sparse model. We refer to the above phenomenon as the \textbf{``reconstruction error explosion''}. This phenomenon can be observed on the left side of Figure \ref{fig:framework}.

From the above theoretical analysis, we understand that the earlier layers are more important than the later layers. Setting a lower sparsity rate for the previous layers helps alleviate the problem of \textbf{``reconstruction error explosion''}. In the next section, we will introduce our method of determining the layer-wise sparsity rate in detail.

\subsection{Determining Layer-wise Sparsity Rates for LLMs}\label{Layer-wiseSparsity}
According to the theorem in Section \ref{sec:error_analysis}, the reconstruction error in sparse LLMs have \textbf{``reconstruction error explosion''} problem. Specifically, the error from the earlier layers will cause an increase in the error of the later layers. When the error from all earlier layers are accumulated, it will lead to a sharp increase in the error of the later layers. Meanwhile, this can lead to an increase in the total reconstruction error of sparse LLMs, thereby damaging the final accuracy of the sparse LLMs. Therefore, in order to mitigate the negative impact of the \textbf{``reconstruction error explosion''} of reconstruction errors on sparse LLMs, we can set the earlier layers to have lower sparsity and the later layers to have higher sparsity. Therefore, we propose to determine the sparsity rate of each layer in LLMs according to the following monotonically increasing arithmetic progression:
\begin{equation}\label{eq:sparsityrates}
s_i = S-\frac{\beta(L-1)}{2}+ \beta \times (i-1),
\quad
i=1,2,\dots,L
\end{equation}
where $s_i$ is the sparsity rate (fraction of zero entries) of the $i$-th layer, $L$ is the total layer number of LLMs, and $S$ is the average sparsity rate of all layers. $\beta$ is a hyperparameter that controls the degree of difference in the sparsity rate of each layer of LLMs. The above formula means that we only need to determine the hyperparameter $\beta$ to get the sparsity rate of each layer of LLMs.

We use grid search \citep{jimenez2008finding} to determine $\beta$ for sparse LLMs. Specifically, since $0\le s_0, s_L\le 1$ and considering the relatively low sparsity rate set for the earlier layers, the arithmetic progression should be increasing. Therefore, we can deduce that the possible range of values for $\beta$ is $0 < \beta \le \min(\frac{2S}{L-1}, \frac{2(1-S)}{L-1})$. This is a small range for a sparse LLMs. For example, for a LLaMA3-8B \citep{llama3} model with 32 layers and the average sparsity rate is $S=0.7$, then the range is $0 < \beta \le 0.019$. In order to find the optimal value of $\beta$ within the above  range, we adopt an grid search method with a step size of $0.002$. The goal is to find the value of $\beta$ that minimizes perplexity of the sparse LLMs on the WikiText-2 \citep{merity2016pointer} dataset. Since the reasonable range of $\beta$ is very small, this ensures that we can find the optimal $\beta$ very quickly. For example, for $0 < \beta \leq 0.019$, we only need to make $9$ attempts. Even for the largest 70B model, the reasonable range of $\beta$ is $0 < \beta \leq 0.0075$, and only $3$ attempts are required in this case. We present our ATP approach in Algorithm \ref{alg}.

Although the above method of determining the layer-wise sparsity rate of LLMs by a monotonically increasing arithmetic progression is very simple, our theoretical analysis in Sec. \ref{sec:AnalysisSparsityMethod} fully proves its rationality, and we have fully demonstrated through a large number of experiments in Sec. \ref{sec:experiments} that our method can effectively improve the accuracy of existing post-training sparsity methods and significantly outperforms current layer-wise sparsity methods.

\subsection{Analysis of the Proposed Determining Sparsity Method}\label{sec:AnalysisSparsityMethod}
We propose the following theorem to prove that the method for determining the layer-wise sparsity rate proposed in Eq. \ref{eq:sparsityrates} is theoretically close to the optimal solution:
\begin{theorem}\label{theorem4}
The total reconstruction error obtained from the monotonically increasing sparsity scheme proposed in Eq. \ref{eq:sparsityrates} is strictly less than that obtained from any non-monotonically increasing sparsity scheme.
\end{theorem}
The proof of the above theorem is detailed in Appendix \ref{sec:provetheorem4}. 

In addition, we compare the layer-wise sparsity rates determined by Eq. \ref{eq:sparsityrates} with those obtained by Bayesian search in Sec. \ref{sec:AblationStudy}. The experimental results show that our method is close to the optimal solution obtained by Bayesian search. Moreover, since our method doesn't require a long time search, it is highly efficient compared to Bayesian search.

All in all, the above theoretical analysis and experimental validation demonstrate the superiority of our method, indicating that the sparsity allocation scheme proposed in Eq. \ref{eq:sparsityrates} is near-optimal both theoretically and empirically.

%% file: exp.tex
\section{Experiments}\label{sec:experiments}
\subsection{Experimental Setup}

\paragraph{Models.} We evaluate ATP across a diverse range of widely-used LLMs, including LLaMA1 (7B, 13B, 30B, 65B) \citep{touvron2023llama}, LLaMA2 (7B, 13B, 70B) \citep{touvron2023llama2}, LLaMA3-8B \citep{llama3}, LLaMA3.1-8B \citep{llama31}, OPT-13B \citep{zhang2022opt}, Vicuna-13B \citep{chiang2023vicuna}, Qwen2.5-7B \citep{yang2024qwen2}, and Mistral-7B \citep{jiang2023mistral}.  
\paragraph{Evaluation.} Our evaluation protocol aligns with established sparsification methods for LLMs (\emph{e.g.,}, SparseGPT \citep{frantar2023sparsegpt} and Wanda \citep{sun2023simple}), encompassing both zero-shot learning and language modeling capabilities. Specifically, we assess the perplexity of the models on the validation set of WikiText-2 \citep{merity2016pointer} and evaluate zero-shot performance on seven downstream tasks: BoolQ \citep{clark2019boolq}, ARC Easy and Challenge \citep{clark2018think}, HellaSwag \citep{zellers2019hellaswag}, WinoGrande \citep{sakaguchi2021winogrande}, OpenbookQA \citep{mihaylov2018can}, and PIQA \citep{bisk2020piqa}. Additionally, we measure performance on arithmetic and knowledge reasoning benchmarks, including 8-shot accuracy on the GSM8K dataset \citep{cobbe2021training} and 5-shot accuracy on the MMLU dataset \citep{hendrycks2020measuring}.  
\paragraph{Baselines.} We apply the layer-wise sparsity rates determined by ATP to several state-of-the-art post-training sparsification methods, including SparseGPT \citep{frantar2023sparsegpt}, Wanda \citep{sun2023simple}, DSnoT \citep{zhang2023dynamic}, Pruner-zero \citep{dong2024pruner}, and ALPS \citep{meng2024alps}. Furthermore, we compare ATP with recent methods for determining layer-wise sparsity rates in LLMs, such as OWL \citep{yin2023outlier}, AlphaPruning \citep{lu2024alphapruning}, DSA \citep{li2024discovering}, and ALS \citep{li2024adaptive}.  
\paragraph{More Models, Evaluations and Baselines.} In Sec. \ref{Moreresults}, we present additional experimental results, including the application of ATP to multimodal and vision models, integration with other compression techniques and LoRA fine-tuning, and comparisons with an expanded set of layer-wise sparsity baselines.  
\paragraph{Implementation Details.} Our pruning implementation builds upon the methods used by SparseGPT and Wanda, with the primary modification being the integration of layer-wise sparsity rates generated by ATP. 

\subsection{Zero-shot Tasks}
\paragraph{Quantitative Evaluation.}  
We report the average performance of 70$\%$ sparse LLMs across seven zero-shot tasks in Table \ref{tb:Zero-shot_tasks_results}. The results demonstrate that our ATP method consistently improves accuracy compared to the uniform sparsity baseline and significantly outperforms other state-of-the-art layer-wise sparsity methods. For instance, with the LLaMA3-8B model pruned using the Wanda method, ATP achieves a 3.43$\%$ higher accuracy than the best-performing DSA method, highlighting the effectiveness and superiority of ATP in enhancing sparse LLM performance.  
\begin{table*}[t]
    \centering
    \caption{Comparison of the average zero-shot accuracy across 7 tasks for 70$\%$ sparse LLMs obtained using various sparsity methods.}
    \label{tb:Zero-shot_tasks_results}
    \resizebox{0.9\textwidth}{!}{
    \begin{tabular}{c|c|cccc|ccc|c}
        \toprule
         & & \multicolumn{4}{c|}{LLaMA}  & \multicolumn{3}{c|}{LLaMA2} & \multicolumn{1}{c}{LLaMA3}\\
        \multirow{-2}{*}{Method} & \multirow{-2}{*}{\makecell[c]{Layerwise\\sparsity}} & 7B & 13B & 30B & 65B & 7B & 13B & 70B & 8B \\
        \midrule 
        Dense & - & 61.74 & 63.84 &  67.41&68.57 & 61.88& 65.00& 69.14 &  65.62  \\ 
        \midrule 
        & Uniform & 37.45 &  40.79& 53.35 & 57.76 &  35.33& 38.88 & 58.48  & 35.42 \\
        & OWL &44.19 & 48.18 & 55.88 & 59.79 &  41.75&   47.19&  59.26&35.42 \\
        & DSA & 43.72& 48.64 & 55.21 & 58.24 & 36.55 & 43.36 &  58.25&  37.85 \\
         &  AlphaPruning&44.99 & 49.81 &   56.67  & 61.05&43.37 &49.11 &   59.44& 35.61\\
        \gr \multirow{-5}{*}{Wanda}&  \bf{ATP} & \bf{47.03} &   \bf{51.60} & \bf{57.69}  &  \bf{61.92} &   \bf{45.83} &  \bf{52.11} &  \bf{60.91} &   \bf{41.28}   \\
        \midrule 
        & Uniform  &43.60 &48.00 &53.64&60.18 & 43.07   &  47.38 & 60.84 & 43.02\\
        & OWL & 46.57& 50.01 & 55.73 & 59.47 & 46.55 &  49.90 & 60.83 & 46.06\\
         &  AlphaPruning&46.58 & 50.63 &56.42  &60.32  & 46.20 & 50.86 & 61.03 & 45.27 \\
        \gr \multirow{-5}{*}{SparseGPT} &  \bf{ATP} & \bf{47.37} &   \bf{52.12} & \bf{57.55}  &  \bf{61.31} &   \bf{48.63} &  \bf{52.46} &  \bf{62.25} &   \bf{47.79} \\
        \bottomrule
    \end{tabular}
    }
\end{table*}
\paragraph{Varying Sparsity Rates.}  
We also evaluate the performance of sparse LLMs under reduced sparsity constraints. Specifically, Table \ref{tab:v_sparsity_rate_Zero-shot} presents the zero-shot accuracy of LLMs pruned using the Wanda method at a 50$\%$ sparsity rate. Even under this lower sparsity setting, ATP demonstrates substantial improvements in accuracy across all models, maintaining its performance advantage over existing layer-wise sparsity methods. This indicates that ATP is robust and effective in optimizing LLM accuracy under varying sparsity rates.  
\begin{table}[h]
\centering
\renewcommand{\arraystretch}{1.1}
\caption{Average zero-shot accuracy of sparse LLaMA2-7B/13B models pruned using the Wanda method at 50$\%$ sparsity rate.}
\label{tab:v_sparsity_rate_Zero-shot}
\resizebox{0.45\textwidth}{!}{
\begin{threeparttable}
\begin{tabular}{l|c|c|c|c}
\toprule 
Method& \multicolumn{1}{c|}{1-7B} & \multicolumn{1}{c|}{1-13B} & \multicolumn{1}{c|}{2-7B} & \multicolumn{1}{c}{2-13B} \\
    \midrule
    Dense &66.09 & 68.18 & 66.69& 69.25 \\
    \midrule
   Uniform &60.06 &  65.22& 63.03&65.73  \\
    OWL & 61.92& 65.91 & 63.88&68.02 \\
    DSA &61.90 & 65.40 &63.89 & 67.65 \\
    ALS& 61.59& 65.05 & 64.12& 67.11 \\
    AlphaPruning &61.91 & 66.19 &64.20 &67.80\\
  \gr   \bf{ATP}  & \bf62.72 &\bf 66.39&\bf64.49  & \bf68.18   \\ 
     \bottomrule
\end{tabular}
\begin{tablenotes}
\normalsize
\item * The zero-shot evaluation setting used here follows the configuration outlined in the ALS paper. For more details, refer to Sec. \ref{Zero-shotevaluationsetting}.
\end{tablenotes}
\end{threeparttable}
}
\end{table}

\subsection{Language Modeling}
We present the perplexity of 50$\%$ to 80$\%$ sparse LLaMA-7B and LLaMA2-7B models pruned using the Wanda method on the WikiText-2 dataset in Table \ref{tab:v_sparsity_rate_ppl}. The results show that our ATP method achieves lower perplexity compared to other layer-wise sparsity methods. Furthermore, this advantage becomes increasingly significant at higher sparsity rates. 
\begin{table}[h]
\centering
\renewcommand{\arraystretch}{1.1}
\caption{WikiText-2 perplexity of sparse LLaMA-7B/2-7B obtained by Wanda across varying sparsity rates.}
\label{tab:v_sparsity_rate_ppl}
\resizebox{0.7\textwidth}{!}{
\begin{tabular}{l| cccc|cccc}
\toprule 
& \multicolumn{4}{c|}{LLaMA-7B} & \multicolumn{4}{c}{LLaMA2-7B} \\
\midrule
   Method & 50\% & 60\%  & 70\% &  80\% & 50\% & 60\% & 70\% & 80\% \\
    \midrule
   Uniform & 7.26   & 10.63  & 84.52  & 5889.13 & 6.92 &10.96&74.26 &1980.85  \\
    OWL& 7.22   & 9.35   & 24.56  & 1002.87&6.87 & 9.80& 30.38&629.99  \\
    DSA &7.17 &9.38 &  24.52 &1232.88 &7.05&10.40&63.71 &1638.81  \\
    AlphaPruning & 7.18   & 9.47   & 23.86 & 698.56& 6.88 & 9.78 & 28.87& 1672.49 \\
  \gr   \bf{ATP} &\textbf{7.05}&\textbf{9.25}&\textbf{20.16}&\textbf{176.80}& \bf 6.82 & \bf 9.15& \bf 22.16 & \bf 425.12  \\ 
     \bottomrule
\end{tabular}}
\end{table}

\subsection{More LLM Architectures}
We validate the effectiveness of our ATP method on LLMs with a broader range of architectures. Specifically, we combined ATP with the Wanda method to obtain 70$\%$ sparse models, including LLaMA3.1-8B, OPT-6.7B, Vicuna-7B, Qwen2.5-7B, and Mistral-7B. We report the experimental results in Table \ref{tab:MoreArchitectures}. The results demonstrate that our method consistently enhances the performance of LLMs across various architectures.

\subsection{More Results}\label{sec:MoreResults}
We provide more experimental results in the appendix. Specifically, in Sec. \ref{sec:arithmetic}, we demonstrate the enhanced performance of our ATP method on arithmetic and knowledge reasoning tasks for sparse LLMs. Additionally, in Secs. \ref{sec:multimodal_tasks} and \ref{sec:vision_models}, we present the performance gains of ATP on sparse multimodal and vision models, respectively. In Sec. \ref{sec:other_Compression_Technologies}, we highlight the performance improvements when integrating ATP with other compression techniques, including N:M sparsity, structured pruning, and quantization. Furthermore, we compare ATP with additional layer-wise sparsity baselines (Sec. \ref{sec:more_baselines}), demonstrate its performance enhancements across various post-training sparsity methods (Sec. \ref{sec:more_post-training}), and showcase its effectiveness when combined with LoRA fine-tuning for sparse LLMs (Sec. \ref{sec:integrate_LoRA}).

\subsection{Ablation Study}\label{sec:AblationStudy}
\paragraph{Searching Step.}
In Sec. \ref{Layer-wiseSparsity}, we perform a grid search with a step size of $0.002$ to determine the optimal value of $\beta$. Here, we analyze the impact of different step sizes on the search results. As shown in Table \ref{tab:SearchingStep}, searches conducted with larger step sizes yield inferior results compared to those with a step size of $0.002$. This is because larger step sizes fail to sufficiently explore the possible optimal values of $\beta$. Conversely, further reducing the step size for a more fine-grained search shows limited improvement in perplexity. Therefore, to balance both accuracy and efficiency, we adopt a grid search with a step size of $0.002$.
\begin{table}[h!]
    \centering
    \caption{The impact of different step sizes on search results.}\label{tab:SearchingStep}
    \resizebox{0.6\textwidth}{!}{
    \begin{tabular}{lccccc}
        \toprule
        \textbf{Step size} & \textbf{0.0005} &\textbf{0.001}& \textbf{0.002}&\textbf{0.004} & \textbf{0.008}  \\
        \midrule
        \textbf{Perplexity (↓)} & 22.14 & 22.16 & 22.16 &23.09  & 23.09 \\
        \textbf{Search Time (min)} & 76 &38  & 18 & 8 &4 \\
        \bottomrule
    \end{tabular}
    }
\end{table}

\paragraph{Comparison with Bayesian Search.}
We compare the ATP method with layer-wise sparsity rates obtained through Bayesian search. Specifically, we use Bayesian search to determine the sparsity rates for the 50$\%$, 60$\%$, and 70$\%$ sparse LLaMA2-7B model. The search is conducted using the Optuna hyperparameter optimization framework \citep{akiba2019optuna}, performing 1000 iterations to optimize the sparsity rates across all 32 layers of LLaMA2-7B. The sparsity rate for each layer is constrained between 0 and 1, while ensuring that the average sparsity matches the target rates. We integrate Bayesian search with the Wanda method, with the objective of minimizing model perplexity on the WikiText-2 dataset. The comparison of perplexity and zero-shot accuracy between sparse LLMs obtained using ATP and Bayesian search is presented in Table \ref{tab:BayesianSearch}. The results indicate that the performance of our ATP method is comparable to the optimal solution obtained through Bayesian search, demonstrating that the layer-wise sparsity rates determined by our method are experimentally close to the optimal values identified by the search approach. However, Bayesian search requires approximately 33 hours to complete on a single NVIDIA A100 80GB GPU, whereas ATP only takes 18 minutes, demonstrating significantly higher efficiency.  
\begin{table}[h!]
    \centering
    \caption{Comparison of ATP and Bayesian search.}\label{tab:BayesianSearch}
    \resizebox{0.6\textwidth}{!}{
    \begin{tabular}{lccc}
        \toprule
        \textbf{Method} &\textbf{Sparsity}& \textbf{Perplexity (↓)} & \textbf{Accuracy (↑)} \\
        \midrule
        Bayesian Search & 50$\%$ & 6.81 & 59.67 \\
        \gr \textbf{ATP}& \textbf{50$\%$} & \textbf{6.82} & \textbf{59.63} \\
        \midrule
        Bayesian Search & 60$\%$ & 9.20 & 54.73 \\
        \gr \textbf{ATP}& \textbf{60$\%$} & \textbf{9.15} & \textbf{54.79} \\
        \midrule
        Bayesian Search & 70$\%$ & 22.10 & 45.90 \\
        \gr \textbf{ATP}& \textbf{70$\%$} & \textbf{22.16} & \textbf{45.83} \\
        \bottomrule
    \end{tabular}
    }
\end{table}

\paragraph{Inference Speedup.} We evaluate the acceleration performance of the sparse LLaMA2-7B model, with results summarized in Table \ref{tab:speedup}. The end-to-end token generation time was measured using the DeepSparse \citep{deepsparse} inference engine on an Intel Xeon Silver 4314 CPU and the nm-vllm \citep{nm-vllm} inference engine on an NVIDIA RTX 4090 GPU. Our method achieves significant speedups, ranging from 1.79$\times$ to 2.63$\times$ on the CPU and 1.71$\times$ to 2.23$\times$ on the GPU, compared to the dense model, at sparsity rates between 50$\%$ and 70$\%$.
\begin{table}[h!]
\centering
\caption{End-to-end inference acceleration of sparse LLaMA2-7B on CPU and GPU.}
\label{tab:speedup}
\setlength{\tabcolsep}{10pt}
\resizebox{0.6\textwidth}{!}{
\begin{tabular}{c|c|c|c|c|c}
\toprule
\textbf{Device} & \textbf{Sparsity} & \textbf{Dense} & \textbf{50$\%$} & \textbf{60$\%$} & \textbf{70$\%$} \\
\midrule
\multirow{2}{*}{CPU} & \makecell[c]{Throughput\\(tokens/s)} $\uparrow$ & 3.40  & \textbf{6.10} & \textbf{7.39} & \textbf{8.95} \\ 
\cmidrule{2-6}
& \makecell[c]{\textbf{Speedup} $\uparrow$} & 1.00$\times$  & \textbf{1.79$\times$} & \textbf{2.17$\times$} & \textbf{2.63$\times$}  \\
\midrule
\multirow{2}{*}{GPU} & \makecell[c]{Throughput\\(tokens/s)} $\uparrow$ & 57.29 & \textbf{97.92} & \textbf{111.86} & \textbf{127.67} \\ 
\cmidrule{2-6}
& \makecell[c]{\textbf{Speedup} $\uparrow$} & 1.00$\times$ & \textbf{1.71$\times$} & \textbf{1.95$\times$} & \textbf{2.23$\times$}  \\
\bottomrule
\end{tabular}}
\end{table}

\paragraph{More Ablation Study.} We provide more ablation results in the appendix. Specifically, in Sec. \ref{sec:ComputationalEfficiency}, we demonstrate the computational efficiency of our ATP method. In Secs. \ref{sec:AnalyzeDistribution} and \ref{ComparisonDistribution}, we analyze the layer-wise sparsity distribution generated by ATP and compare it with the distributions produced by other layer-wise sparsity methods. Furthermore, in Sec. \ref{sec:Different_beta}, we explore the impact of different $\beta$ settings on the perplexity of sparse LLMs.

%% file: conclusion.tex
\section{Conclusion}
In this paper, we propose a theoretically grounded method for determining layer-wise sparsity rates in LLMs, effectively addressing the challenge of \textbf{``reconstruction error explosion''}. Our approach utilizes an arithmetic progression to streamline sparsity allocation, reducing the complexity to a single hyperparameter while achieving near-optimal performance with minimal tuning. We have demonstrated through theoretical analysis and experimental validation that our sparsity allocation scheme is close to the optimal solution. Extensive experiments demonstrate the effectiveness of our method, yielding significant improvements in model perplexity, accuracy, and inference speed. Furthermore, our approach exhibits strong generalization across diverse architectures and modalities, establishing it as a versatile and robust solution for optimizing compressed models.

%% file: appendix.tex
\section{Limitation and Future Work}
In this paper, we derive that layer-wise sparsity rates should gradually increase based on the reconstruction error analysis and determine these rates using a monotonically increasing arithmetic progression. Extensive experiments validate the effectiveness of our method. However, the arithmetic progression configuration may not be optimal, and we plan to explore more diverse sparsity rate schemes in the future. Additionally, while our method significantly enhances the accuracy of existing post-training sparsity techniques for LLMs, there remains a performance gap compared to lossless sparsification, particularly under high sparsity conditions. In future work, we aim to develop more advanced post-training sparsity methods to further improve the accuracy of sparse LLMs.

\section{ATP Algorithm}
\begin{algorithm}[H]
\caption{Using ATP method to obtain sparse LLMs}
\label{alg}
\KwIn{Dense LLMs $\mathcal{M}_{dense}$, average sparsity rate $S$}
\KwOut{Sparse LLMs $\mathcal{M}_{sparse}$}
Use grid search to get $\beta$ in Eq. \ref{eq:sparsityrates};

Determine layer-wise sparsity rate according to Eq. \ref{eq:sparsityrates};

Combine with post-training sparsity method (\emph{e.g.,} SpaeseGPT or Wanda) to obtain sparsity mask $\mathbf{W}$;

Apply $\mathbf{W}$ to $\mathcal{M}_{dense}$ yields $\mathcal{M}_{sparse}$.
\end{algorithm}

\section{Proof of Lemma \ref{lemma1}}\label{sec:provelemma}
\begin{proof}[Proof of Lemma \ref{lemma1}]
Let \(\boldsymbol{A} \in \mathbb{R}^{m \times n}\) and \(\boldsymbol{B} \in \mathbb{R}^{n \times p}\) be any matrices. Consider the singular value decomposition (SVD) of \(\boldsymbol{A}\):
\begin{equation}
\boldsymbol{A} = \boldsymbol{U} \boldsymbol{\Sigma} \boldsymbol{V}^T
\end{equation}
where:
\begin{itemize}
    \item \(\boldsymbol{U} \in \mathbb{R}^{m \times m}\) is an orthogonal matrix (\(\boldsymbol{U}^T \boldsymbol{U} = \boldsymbol{I}\)),
    \item \(\boldsymbol{V} \in \mathbb{R}^{n \times n}\) is an orthogonal matrix (\(\boldsymbol{V}^T \boldsymbol{V} = \boldsymbol{I}\)), and
    \item \(\boldsymbol{\Sigma} \in \mathbb{R}^{m \times n}\) is a diagonal matrix with non-negative singular values \(\sigma_1 \geq \sigma_2 \geq \dots \geq \sigma_{\min(m,n)} \geq 0\) on the diagonal.
\end{itemize}

Substituting the SVD of \(\boldsymbol{A}\) into the expression for \(\boldsymbol{AB}\), we have:
\begin{equation}
\boldsymbol{AB} = \boldsymbol{U} \boldsymbol{\Sigma} \boldsymbol{V}^T \boldsymbol{B}
\end{equation}

Therefore,
\begin{equation}
\left\lVert \boldsymbol{AB} \right\rVert_F^2 = \left\lVert \boldsymbol{U} \boldsymbol{\Sigma} \boldsymbol{V}^T \boldsymbol{B} \right\rVert_F^2
\end{equation}

Since \(\boldsymbol{U}\) is an orthogonal matrix, the Frobenius norm is invariant under orthogonal transformations. Specifically:
\begin{equation}
\left\lVert \boldsymbol{U} \boldsymbol{X} \right\rVert_F = \left\lVert \boldsymbol{X} \right\rVert_F \quad \forall \ \boldsymbol{X}
\end{equation}

Applying this property:
\begin{equation}
\left\lVert \boldsymbol{AB} \right\rVert_F^2 = \left\lVert \boldsymbol{\Sigma} \boldsymbol{V}^T \boldsymbol{B} \right\rVert_F^2
\end{equation}

The matrix \(\boldsymbol{\Sigma}\) is diagonal with singular values \(\sigma_i\) on the diagonal. Let \(\sigma_{\min} = \sigma_{\min}(\boldsymbol{A})\) denote the smallest singular value of \(\boldsymbol{A}\). Then, for any matrix \(\boldsymbol{X}\), we have:
\begin{equation}
\boldsymbol{\Sigma} \boldsymbol{X} \geq \sigma_{\min} \boldsymbol{X}
\end{equation}
in the sense that each singular value scales the corresponding component of \(\boldsymbol{X}\).

Therefore, applying this to our case:
\begin{equation}
\left\lVert \boldsymbol{\Sigma} \boldsymbol{V}^T \boldsymbol{B} \right\rVert_F^2 \geq \sigma_{\min}^2 \left\lVert \boldsymbol{V}^T \boldsymbol{B} \right\rVert_F^2
\end{equation}
This inequality holds because scaling each component by at least \(\sigma_{\min}\) results in the squared norm being scaled by at least \(\sigma_{\min}^2\).

Similarly to \(\boldsymbol{U}\), the orthogonal matrix \(\boldsymbol{V}\) preserves the Frobenius norm:
\begin{equation}
\left\lVert \boldsymbol{V}^T \boldsymbol{X} \right\rVert_F = \left\lVert \boldsymbol{X} \right\rVert_F \quad \forall \ \boldsymbol{X}
\end{equation}

Applying this property:
\begin{equation}
\left\lVert \boldsymbol{V}^T \boldsymbol{B} \right\rVert_F = \left\lVert \boldsymbol{B} \right\rVert_F
\end{equation}

Substituting back, we obtain:
\begin{equation}
\left\lVert \boldsymbol{AB} \right\rVert_F^2 = \left\lVert \boldsymbol{\Sigma} \boldsymbol{V}^T \boldsymbol{B} \right\rVert_F^2 \geq \sigma_{\min}^2 \left\lVert \boldsymbol{V}^T \boldsymbol{B} \right\rVert_F^2 = \sigma_{\min}^2 \left\lVert \boldsymbol{B} \right\rVert_F^2
\end{equation}

Thus, we have shown that:
\begin{equation}
\left\lVert \boldsymbol{AB} \right\rVert_F^2 \geq \sigma_{\min}^2 (\boldsymbol{A}) \left\lVert \boldsymbol{B} \right\rVert_F^2
\end{equation}
\end{proof}

\section{Proof of Theorem \ref{theorem4}}\label{sec:provetheorem4}
\begin{proof}[Proof of Theorem \ref{theorem4}]
Recall that each layer $i \in \{1, 2, \dots, L\}$ in an LLM has a sparsity rate $s_i \in [0, 1]$, where $s_i$ indicates the fraction of zero entries in the weight matrix of layer $i$. Let $\mathcal{L}_i$ denote the reconstruction error of layer $i$, and let $\mathcal{L} = \sum_{i=1}^{L} \mathcal{L}_i$ be the total reconstruction error of the entire LLM.

To prove Theorem \ref{theorem4}, we will make the following assumptions based on the preceding sections:

\begin{enumerate}
    \item According to Theorem \ref{theorem1}, The reconstruction error for each layer \( i \), denoted as \( \mathcal{L}_i \), is an increasing function of the sparsity rate \( s_i \), denote as \( f(s_i)\). 
    
    \item According to Theorem \ref{theorem2}, the reconstruction error propagates through layers such that the increase of reconstruction error of \(i\)-th layer will lead to the increase of reconstruction error of \((i+1)\)-th layer. Therefore we assume \( \mathcal{L}_{i+1} = c \mathcal{L}_i+f(s_i) \), where \( c > 1 \). The above formula indicates that the reconstruction error of a layer not only accumulates the reconstruction errors from previous layers but also that the current sparse layer contributes to the reconstruction error.
\end{enumerate}

The monotonically increasing sparsity scheme given by Eq.~\ref{eq:sparsityrates} (in Sec.~\ref{Layer-wiseSparsity}) implies that \(s_1 < s_2 < \cdots < s_L\), let us denote such a scheme as
\begin{equation}
\mathbf{s}^\uparrow \;=\; (s_1^\uparrow,\, s_2^\uparrow,\, \dots,\, s_L^\uparrow),
\end{equation}
where $s_1^\uparrow < s_2^\uparrow < \cdots < s_L^\uparrow$.

Consider a different non-monotonically increasing sparsity scheme:
\begin{equation}
  \mathbf{s}^\diamond \;=\; (s_1^\diamond,\, s_2^\diamond,\, \dots,\, s_L^\diamond),
\end{equation}
which is \textit{not} monotonically increasing. That is, there exists at least one index $k$ such that \(s_k^\diamond \;>\; s_{k+1}^\diamond\)
  
Moreover, suppose both $\mathbf{s}^\uparrow$ and $\mathbf{s}^\diamond$ satisfy the constraint that their average sparsities match the same overall budget. Formally,
\begin{equation}
  \frac{1}{L} \sum_{i=1}^L s_i^\uparrow \;=\; \frac{1}{L} \sum_{i=1}^L s_i^\diamond \;=\; S,
\end{equation}
and each $\mathbf{s}^\uparrow$ and $\mathbf{s}^\diamond$ results in a total reconstruction error
\begin{equation}
  \mathcal{L}^\uparrow \;=\; \sum_{i=1}^{L} \mathcal{L}_i^\uparrow, 
  \quad
  \mathcal{L}^\diamond \;=\; \sum_{i=1}^{L} \mathcal{L}_i^\diamond.
\end{equation}

We will show that for any non-monotonic sparsity vector $\mathbf{s}^\diamond$, we can \textit{redundantly reorder} it layer by layer to get a monotonically increasing vector $\mathbf{s}^\uparrow$ of the same average sparsity, such that the overall reconstruction error $\mathcal{L}^\uparrow$ is \textit{strictly smaller} than $\mathcal{L}^\diamond$. 
Ultimately, we will deduce
\begin{equation}
  \mathcal{L}^\uparrow < \mathcal{L}^\diamond.
\end{equation}

Let us focus on any adjacent pair $(s_k^\diamond, s_{k+1}^\diamond)$ where $s_k^\diamond > s_{k+1}^\diamond$. Define $s_k^\ast = s_{k+1}^\diamond$ and $s_{k+1}^\ast = s_k^\diamond$, effectively \textit{swapping} these two sparsities. We then compare the total contribution of layers $k$ and $k+1$ to the overall reconstruction error, first in the non-monotonic case and then in the swapped case.

We restrict our attention to layers $k$ and $k+1$:
\begin{equation}
  \text{Layer } k \quad\Rightarrow\quad 
  \mathcal{L}_k^\diamond = f(s_k^\diamond),
  \qquad
  \text{Layer } k+1 \quad\Rightarrow\quad
  \mathcal{L}_{k+1}^\diamond = c \, \mathcal{L}_k^\diamond + f(s_{k+1}^\diamond).
\end{equation}

Strictly speaking, $\mathcal{L}_{k+1}^\diamond$ depends on partial errors from layer~$k$ and its own sparsity $s_{k+1}^\diamond$. Intuitively, a higher error $\mathcal{L}_k^\diamond$ “propagates” or “magnifies” into layer $k+1$. We keep the rest of the layer-wise sparsities fixed outside of these two positions to isolate their effect.

Now, consider \emph{swapping} the two sparsities: $(s_k^\diamond,\, s_{k+1}^\diamond) \;\mapsto\; (s_k^\ast,\, s_{k+1}^\ast) = (s_{k+1}^\diamond,\, s_k^\diamond)$. Let
\begin{equation}
  \mathcal{L}_k^\ast = f(s_k^\ast) = f(s_{k+1}^\diamond),
  \quad
  \mathcal{L}_{k+1}^\ast = c \, \mathcal{L}_k^\ast + f(s_{k+1}^\ast) = c\,f(s_{k+1}^\diamond) + f(s_k^\diamond).
\end{equation}

Therefore, the total for these two layers is
\begin{equation}
  \mathcal{L}_\text{swapping}= \mathcal{L}_k^\ast + \mathcal{L}_{k+1}^\ast 
  \;=\; f(s_{k+1}^\diamond) 
  + \Bigl(c\,f(s_{k+1}^\diamond) + f(s_k^\diamond)\Bigr)
  \;=\; (1 + c)\,f(s_{k+1}^\diamond) \;+\; f(s_k^\diamond).
\end{equation}
Meanwhile, the original total is
\begin{equation}
  \mathcal{L}_\text{original}= \mathcal{L}_k^\diamond + \mathcal{L}_{k+1}^\diamond 
  \;=\; f(s_k^\diamond) 
  \;+\; \Bigl(c\,f(s_k^\diamond) + f(s_{k+1}^\diamond)\Bigr)
  \;=\; (1 + c)\,f(s_k^\diamond) \;+\; f(s_{k+1}^\diamond).
\end{equation}
The difference between the two is
\begin{equation}
\begin{aligned}
   \mathcal{L}_\text{original}-\mathcal{L}_\text{swapping}
   &=\Bigl[(1 + c)\,f(s_k^\diamond) + f(s_{k+1}^\diamond)\Bigr]
    \;-\;\Bigl[(1 + c)\,f(s_{k+1}^\diamond) + f(s_k^\diamond)\Bigr]\\
   &=\;(1 + c)\bigl[f(s_k^\diamond) - f(s_{k+1}^\diamond)\bigr]
      \;-\;\bigl[f(s_k^\diamond) - f(s_{k+1}^\diamond)\bigr]\\
   &=\;c\,\bigl[f(s_k^\diamond) - f(s_{k+1}^\diamond)\bigr].
\end{aligned}
\end{equation}
Since $c > 1$ and $f(s_k^\diamond) > f(s_{k+1}^\diamond)$, it follows that \(c\,\bigl[f(s_k^\diamond) - f(s_{k+1}^\diamond)\bigr] \;>\; 0.\)

Hence we obtain
\begin{equation}
  \mathcal{L}_\text{original}
  \;>\; \mathcal{L}_\text{swapping} .
\end{equation}
This shows that \emph{locally swapping} a pair $(s_k^\diamond, s_{k+1}^\diamond)$ with $s_k^\diamond> s_{k+1}^\diamond$ \emph{reduces} the total reconstruction error contributed by layers $k$ and $k+1$. Globally across all $L$ layers, repeating such a swap for each adjacent pair that breaks the monotonicity moves $\mathbf{s}^\diamond$ to a strictly monotonically increasing sequence of sparsities.

By iterating this argument over each adjacent pair that fails monotonicity, we reorder $\mathbf{s}^\diamond$ into  $\mathbf{s}^\uparrow$. Since every swap strictly reduces the local error, the resulting \textit{monotonically increasing} scheme $\mathbf{s}^\uparrow$ has an overall reconstruction error:
\begin{equation}
  \mathcal{L}^\uparrow 
  \;=\; \sum_{i=1}^{L} \mathcal{L}_i^\uparrow
  \;<\; \sum_{i=1}^{L} \mathcal{L}_i^\diamond
  \;=\; \mathcal{L}^\diamond.
\end{equation}
Therefore, the total reconstruction error obtained from the monotonically increasing sparsity scheme proposed in Eq. \ref{eq:sparsityrates} is strictly less than that obtained from any non-monotonically increasing sparsity scheme, under the same average sparsity constraint.

\end{proof}

\section{Zero-shot evaluation setting}\label{Zero-shotevaluationsetting}
We use the lm-eval-harness framework \citep{gao2021framework} to evaluate the zero-shot performance of LLMs. By default, we follow the settings used in Wanda \citep{sun2023simple} and AlphaPruning \citep{lu2024alphapruning}, reporting the ``acc'' metric across all datasets. However, lm-eval-harness provides multiple metrics depending on the dataset, including both ``acc'' and ``acc\_norm''. In contrast, ALS \citep{li2024adaptive} employs different metrics for different datasets. The evaluation metrics are summarized in Table \ref{tab:evaluation-metrics}.
\begin{table}[h]
    \centering
    \caption{Evaluation Metrics for Different Datasets}
    \label{tab:evaluation-metrics}
    \resizebox{0.3\textwidth}{!}{
    \begin{tabular}{l|l}
        \hline
        \textbf{Dataset}     & \textbf{Metric} \\ \hline
        boolq               & acc              \\ \hline
        hellaswag           & acc\_norm        \\ \hline
        arc\_easy           & acc              \\ \hline
        piqa                & acc              \\ \hline
        arc\_challenge      & acc\_norm        \\ \hline
        winogrande          & acc              \\ \hline
        openbookqa          & acc\_norm        \\ \hline
    \end{tabular}
    }
\end{table}

\section{More Results}\label{Moreresults}
\subsection{More LLM Architectures}
\begin{table}[h]
    \centering
    \caption{The performance improvement of ATP on sparse LLMs across a wider range of architectures.}\label{tab:MoreArchitectures}
    \resizebox{0.48\textwidth}{!}{
    \begin{tabular}{lccc}
        \toprule
        \textbf{Method} & \textbf{Model} & \textbf{Perplexity (↓)} & \textbf{Accuracy (↑)} \\
        \midrule
        Dense & LLaMA3.1-8B& 6.18  & 65.93 \\
        Wanda  & LLaMA3.1-8B & 109.99 &  36.10\\
        \gr \textbf{w. ATP} & \textbf{LLaMA3.1-8B} & \textbf{72.05} & \textbf{41.59} \\
        \midrule
        Dense &OPT-13B&10.13  & 55.22 \\
        Wanda  &OPT-13B& 73.70 &  41.51\\
        \gr \textbf{w. ATP} & \textbf{OPT-13B} & \textbf{33.98} & \textbf{44.58} \\
        \midrule
        Dense & Vicuna-13B&  5.94& 65.53 \\
        Wanda  & Vicuna-13B & 44.89 & 42.06 \\
        \gr \textbf{w. ATP} & \textbf{Vicuna-13B} & \textbf{20.24} & \textbf{53.19} \\
        \midrule
        Dense & Qwen2.5-7B& 6.77 &65.36  \\
        Wanda  & Qwen2.5-7B & 75.16 & 41.10 \\
        \gr \textbf{w. ATP} & \textbf{Qwen2.5-7B} & \textbf{43.28} & \textbf{42.82} \\
        \midrule
        Dense &Mistral-7B& 5.28 & 66.17 \\
        Wanda  & Mistral-7B & 57.44 & 37.62 \\
        \gr \textbf{w. ATP} & \textbf{Mistral-7B} & \textbf{29.19} & \textbf{42.76} \\
        \bottomrule
    \end{tabular}
    }
\end{table}

\subsection{Arithmetic and Knowledge Reasoning Tasks}\label{sec:arithmetic}
We further evaluate the performance improvements of our ATP method on arithmetic and knowledge reasoning tasks for sparse models. Specifically, We evaluate the 8-shot accuracy of 40$\%$ sparse models on the GSM8K dataset and the 5-shot accuracy of 60$\%$ sparse models on the MMLU dataset. The results are presented in Table \ref{tab:gsm8kandmmlu}. Our ATP method significantly improves the accuracy of sparse models on both tasks, further demonstrating the generalization and effectiveness of our approach.
\begin{table}[h!]
    \centering
    \caption{Accuracy of sparse LLaMA2 on the GSM8K and MMLU datasets.}\label{tab:gsm8kandmmlu}
    \resizebox{0.45\textwidth}{!}{
    \begin{tabular}{lcccc}
        \toprule
         & \multicolumn{2}{c}{\textbf{GSM8K}}  & \multicolumn{2}{c}{\textbf{MMLU}} \\
        \textbf{Method} & \textbf{7B} & \textbf{13B} & \textbf{7B} & \textbf{13B}  \\
        \midrule
        Dense         & 14.60 &24.56  & 45.90 &55.70 \\
        \midrule
        SparseGPT     & 11.07 & 16.91 & 29.40  &39.60 \\
        \gr \textbf{SparseGPT w.ATP} &\textbf{12.01}&\textbf{19.02}&\textbf{32.40}&\textbf{47.70}
        \\
        \midrule
        Wanda         & 8.72 &17.51  &  28.90 & 34.80  \\
        \gr \textbf{Wanda w.ATP} &\textbf{9.93}&\textbf{19.56}&\textbf{32.80}&\textbf{46.20}\\
        \bottomrule
    \end{tabular}
    }
\end{table}

\subsection{Multimodal Tasks}\label{sec:multimodal_tasks}
We demonstrate the applicability of our method to multimodal models by integrating it with Wanda to sparsify the Vicuna-7B model \citep{chiang2023vicuna} in LLaVA-1.5 \citep{liu2024improved}. The performance of the sparse model was evaluated on various visual question-answering and reasoning benchmarks, including VQA \citep{singh2019towards}, VQAv2 \citep{goyal2017making}, and SQA \citep{lu2022learn}. We compared our method with magnitude-based pruning, SparseGPT, Wanda, and the DSA method combined with Wanda, under a 50$\%$ sparsity rate. The results in Table \ref{tab:multimodal} show that our method outperforms the magnitude-based, SparseGPT, and Wanda approaches and demonstrates superiority over the DSA method.
\begin{table}[h!]
    \centering
    \caption{The performance improvement of the ATP method on sparse multimodal models.}
    \resizebox{0.4\textwidth}{!}{
    \label{tab:multimodal}
    \begin{tabular}{lccc}
        \toprule
        \textbf{LLaVA-1.5}  & \textbf{VQA} & \textbf{VQAv2} & \textbf{SQA} \\
        \midrule
        Dense & 58.20 & 78.50 & 66.80  \\
        Magnitude & 38.39 & 63.50 & 31.24  \\
        SparseGPT & 53.69 & 75.86 & 63.92 \\
        Wanda & 53.05 & 75.72 & 63.99  \\
        Wanda w. DSA & 54.36& 76.08 & 65.57 \\
       \gr \textbf{Wanda w. ATP} &\textbf{55.21} &\textbf{76.92} &\textbf{66.01} \\
        \bottomrule
    \end{tabular}
    }
\end{table}

\subsection{Vision Models}\label{sec:vision_models}
We compare our ATP method with other approaches for determining layer-wise sparsity rates on vision models. Experiments are conducted on both CNN and Transformer architectures, including ConvNeXt \citep{Liu_2022_CVPR}, ViT \citep{dosovitskiy2021an}, and DeiT \citep{touvron2021training}, and we report the Top-1 accuracy on the ImageNet-1K dataset \citep{deng2009imagenet}. Sparse models are obtained using the Wanda \citep{sun2023simple} method, and ATP is compared against the following baselines: uniform sparsity rate, OWL \citep{yin2023outlier}, and AlphaPruning \citep{lu2024alphapruning}. The results for ConvNeXt are presented in Table \ref{tb:ConvNext}. These results demonstrate that our method is effective for determining layer-wise sparsity rates in vision models and outperforms existing methods. This further confirms the broad applicability of our monotonically increasing sparsity scheme in improving the accuracy of diverse sparse models.
\begin{table}[h!]
    \centering
    \caption{Comparison of ImageNet-1K accuracy across different layer-wise sparsity methods on the ConvNeXt-Base model.}\label{tb:ConvNext}
    \resizebox{0.45\textwidth}{!}{
    \begin{tabular}{lcccc}
        \toprule
        \textbf{Method} & \textbf{50$\%$} &\textbf{60$\%$}& \textbf{70$\%$}&\textbf{80$\%$} \\
        \midrule
        Wanda & 82.72 & 80.55 & 68.18 & 6.44 \\
        w. OWL & 82.76 & 80.53 & 68.28 & 6.32 \\
        w. AlphaPruning & 82.76 & 80.89 &74.24 & 42.35 \\
        \gr \textbf{w. ATP} & \textbf{82.85} & \textbf{81.10} &\textbf{75.58} &\textbf{56.81} \\
        \bottomrule
    \end{tabular}
    }
\end{table}

We also present the accuracy results for sparse ViT and DeiT models in Table \ref{tb:vitanddeit}. Additionally, we compare our method with several widely-used non-uniform layer-wise sparsity strategies in computer vision, including ERK \citep{mocanu2018scalable}, Global \citep{frankle2018lottery}, and LAMP \citep{lee2020layer}, as well as AlphaPruning \citep{lu2024alphapruning}. The results indicate that our method outperforms all the aforementioned baselines.
\begin{table*}[!th]
    \centering
    \caption{Comparison of ImageNet-1K accuracy across different layer-wise sparsity methods on the ViT-B and DeiT-B models.
    }
    \label{tb:vitanddeit} 
    \resizebox{0.6\textwidth}{!}{
    \begin{tabular}{l|ccc|ccc}
        \toprule
        & \multicolumn{3}{c|}{ViT-B 16/224} & \multicolumn{3}{c}{DeiT-B 16/224}\\
        Method & 40\% & 50\% & 60\% & 40\% & 50\% & 60\% \\
        \midrule 
        Uniform & 70.87 & 59.46 & 29.97 & 80.08 & 76.37 & 61.72\\
        ERK \citep{mocanu2018scalable} & 70.89 & 60.49 & 33.15 & 80.05 & 76.22 & 63.49 \\
        Global \citep{frankle2018lottery} & 66.81 & 45.75 & 8.09 & 79.94 & 75.09 & 57.01\\
        LAMP \citep{lee2020layer} & 69.45 & 57.51 & 26.99 & 80.19 & 76.35 & 63.32 \\
        AlphaPruning \citep{lu2024alphapruning} & 71.58 & 64.29 & 44.21 & 80.21 & 77.11 & 64.56 \\
        \gr \textbf{ATP} &\textbf{72.03} &\textbf{65.46} &\textbf{47.74} & \textbf{80.50}&\textbf{78.02} &\textbf{69.73} \\
        \bottomrule
    \end{tabular}
    }
\end{table*}

\subsection{Integrate with other Compression Technologies}\label{sec:other_Compression_Technologies}
To demonstrate the generalization ability of our method for determining layer-wise sparsity rates, we combine ATP with other compression techniques, including N:M sparsity, structured pruning, and mixed-precision quantization. For N:M sparsity, we follow the mixed N:8 settings \citep{sun2021dominosearch}, using ATP to determine the value of \(N\) for each layer while maintaining average sparsity at 2:8, 3:8 and 4:8. In terms of structured pruning, we integrated ATP with LLM-Pruner \citep{ma2023llm}, where ATP determines the layer-wise sparsity rates, and LLM-Pruner applies pruning accordingly. For mixed-precision quantization, we combine ATP with the LIMPQ method \citep{tang2022mixed} to determine the quantization bits for each layer. We apply these compression techniques to the LLaMA-7B model and report the perplexity of the compressed models on the WikiText-2 validation set. The experimental results are presented in Table \ref{tab:otherCompressionTechnologies}. It shows that ATP significantly enhances the performance of various compression methods and outperforms both the OWL and AlphaPruning approaches.
\begin{table}[h!]
    \centering
    \caption{Experimental results of combining ATP with other compression technologies.}\label{tab:otherCompressionTechnologies}
    \resizebox{0.6\textwidth}{!}{
    \begin{tabular}{lccc}
        \toprule
        \textbf{Method} & \textbf{4:8}  & \textbf{3:8}& \textbf{2:8} \\
        \midrule
        Wanda  & 8.57& 42.56& 2962.00  \\
        OWL & 8.55& 22.77 & 331.37\\
        AlphaPruning   & 8.55&21.49 &585.81\\
        \gr \textbf{ATP} & \textbf{8.15}&\textbf{16.08}&\textbf{265.96}\\
        \bottomrule
        \textbf{Method} & \textbf{20$\%$} & \textbf{40$\%$} & \textbf{60$\%$}  \\
        \midrule
        LLM-Pruner &16.95& 30.38 &90.02\\
        OWL & 18.57&28.65 &76.99\\
        AlphaPruning &16.78 & 29.11& 71.21\\
        \gr \textbf{ATP}& \textbf{15.51} &\textbf{27.40}&\textbf{64.25}\\
        \bottomrule
        \textbf{Method} & \textbf{Mixed 3/4 Bit} & \textbf{Mixed 2/3/4 Bit} & \textbf{Mixed 2/4 Bit}\\
        \midrule
        Random   & 12.04 & 11455.54 & 14817.12 \\
        $L_1$ norm & 14.61 & 13959.42 & 33679.21 \\
        OWL      & 9.54  & 311.95   & 8429.39  \\
        AlphaPruning     & 9.01  & 261.39   & 7630.14  \\
        \gr \textbf{ATP}& \textbf{8.52}  & \textbf{198.65}&\textbf{5321.35}\\
        \bottomrule
    \end{tabular}
    }
\end{table}

\subsection{More Layer-wise Sparsity Baselines}\label{sec:more_baselines}
We compare our method with additional approaches for determining layer-wise sparsity rates, including Uniform \citep{zhu2017prune}, Global \citep{frankle2018lottery}, ER \citep{mocanu2018scalable}, ER-Plus \citep{liu2022unreasonable}, and LAMP \citep{lee2020layer}. These methods are combined with the Wanda pruning approach to obtain sparse LLaMA-7B models. The experimental results are presented in Table \ref{tab:MoreBaselines}. It shows that across sparsity rates ranging from 50$\%$ to 80$\%$, the perplexity of models pruned using the ATP method is consistently lower than that of other baselines, further demonstrating the superiority of our approach.
\begin{table}[h!]
    \centering
    \caption{Comparison of WikiText-2 perplexity across various sparsity rates and methods.}\label{tab:MoreBaselines}
    \resizebox{0.5\textwidth}{!}{
    \begin{tabular}{lcccc}
        \toprule
        \textbf{Method/Perplexity (↓)} & \textbf{50\%} & \textbf{60\%} & \textbf{70\%} & \textbf{80\%} \\
        \midrule
        Global  & 14848  & 17765  & 5147   & 39918.56 \\
        LAMP    & 7.57   & 12.86  & 185.52 & 15647.87 \\
        LAMP (per-block)  & 7.25   & 10.95  & 98.77  & 7318.08  \\
        ER      & 7.80   & 12.41  & 119.66 & 6263.79  \\
        ER-Plus & 8.00   & 14.04  & 229.17 & 6013.91  \\
        Uniform & 7.26   & 10.63  & 84.52  & 5889.13  \\
        \gr \textbf{ATP} &\textbf{7.05}&\textbf{9.25}&\textbf{20.16}&\textbf{176.80}\\
        \bottomrule
    \end{tabular}
    }
\end{table}

\subsection{Integrate with More Post-training Sparsity Methods}\label{sec:more_post-training}
We have demonstrated that our method improves performance over Wanda and SparseGPT methods. Notably, ATP can be integrated with any post-training sparsity method to further enhance their effectiveness. In this section, we showcase the performance improvements achieved by combining ATP with other post-training sparsity methods. Specifically, we apply ATP to DSnoT \citep{zhang2023dynamic}, Pruner-Zero \citep{dong2024pruner}, and ALPS \citep{meng2024alps} to obtain a 70$\%$ sparse LLaMA2-7B model. The perplexity and zero-shot accuracy results of the sparse models are presented in Table \ref{tab:more_post-training}. The results indicate that ATP significantly enhances the performance of DSnoT, Pruner-Zero, and ALPS methods.
\begin{table}[h!]
    \centering
    \caption{The performance improvement of ATP on DSnoT, Pruner-zero and ALPS methods.}\label{tab:more_post-training}
    \resizebox{0.4\textwidth}{!}{
    \begin{tabular}{lcc}
        \toprule
        \textbf{Method} & \textbf{Perplexity (↓)} & \textbf{Accuracy (↑)} \\
        \midrule
        Dense  & 5.12 & 61.88 \\
        \midrule
        DSnoT   & 77.83 & 35.11 \\
        \gr \textbf{w. ATP}& \textbf{24.31} & \textbf{44.51} \\
        \midrule
        Pruner-Zero   &103.15  & 34.78 \\
        \gr \textbf{w. ATP}& \textbf{48.82} & \textbf{44.77} \\
        \midrule
        ALPS   &19.31  & 46.75 \\
        \gr \textbf{w. ATP}& \textbf{17.99} & \textbf{49.82} \\
        \bottomrule
    \end{tabular}
    }
\end{table}

\subsection{Integrate with LoRA Fine-tuning}\label{sec:integrate_LoRA}
We further demonstrate the effectiveness of LoRA fine-tuning \citep{hu2021lora} in narrowing the performance gap between highly sparse LLMs and dense models. Specifically, we obtain a 70$\%$ sparse LLaMA2-7B model using the Wanda method and fine-tune it on 10,000 samples from the Alpaca-GPT4 \citep{peng2023instruction} dataset. We compare the results against models sparsified using uniform sparsity, OWL, and AlphaPruning methods. The results in Table \ref{tab:lora} show that LoRA fine-tuning significantly improves the accuracy of the sparse model, further reducing the gap with the dense model. Additionally, the sparse model obtained through the ATP method achieves higher accuracy, and this advantage is retained even after LoRA fine-tuning.
\begin{table}[h!]
    \centering
    \caption{Comparison of the perplexity and zero-shot accuracy of the 70$\%$ sparse LLaMA2-7B obtained by LoRA fine-tuning.}\label{tab:lora}
    \resizebox{0.5\textwidth}{!}{
    \begin{tabular}{lccc}
        \toprule
        \textbf{Method} & \textbf{Fine-tuning} & \textbf{Perplexity (↓)} & \textbf{Accuracy (↑)} \\
        \midrule
        Dense & N.A. &5.12  & 61.88 \\
        \midrule
        Uniform  & N.A. &74.26  & 35.33 \\
        Uniform & LoRA & 13.36 & 47.10 \\
        \midrule
        OWL  & N.A. &30.38  & 41.75 \\
        OWL & LoRA & 12.89 & 50.26 \\
        \midrule
        DSA  & N.A. &63.71  & 36.55 \\
        DSA & LoRA & 13.19 & 49.49 \\
        \midrule
        AlphaPruning & N.A. & 28.87 & 43.37 \\
        AlphaPruning& LoRA & 12.76 &50.37 \\
        \midrule
        \textbf{ATP} &\textbf{N.A.}&\textbf{22.16}&\textbf{45.83}\\
        \gr \textbf{ATP}&\textbf{LoRA}&\textbf{12.28}&\textbf{50.79}\\
        \bottomrule
    \end{tabular}
    }
\end{table}

\section{More Ablation Study}\label{sec:MoreAblationStudy}
\subsection{Computational Efficiency.}\label{sec:ComputationalEfficiency}
In Table \ref{tab:ComputationEfficiency}, we report the time required to determine the layer-wise sparsity rates for 70$\%$ sparse LLMs using our ATP method. The measurements were conducted on NVIDIA A100 80GB GPUs. The results show that only a few searches are needed to obtain the sparsity rates due to the narrow range of reasonable values for the hyperparameter \(\beta\). Furthermore, as the number of model parameters increases, this range narrows even further, enabling the sparsity rates to be determined with fewer searches. For example, for the largest 70B model, only three searches are necessary. This demonstrates that our method is highly computationally efficient.
\begin{table}[h!]
    \centering
    \caption{Computational efficiency of our ATP method (in minutes).}\label{tab:ComputationEfficiency}
    \resizebox{0.5\textwidth}{!}{
    \begin{tabular}{lcccccc}
        \toprule
        \textbf{Model} &\textbf{7B}& \textbf{8B}&\textbf{13B} & \textbf{30B}& \textbf{65B} & \textbf{70B} \\
        \midrule
        Wanda &2.0  & 2.2 & 3.8 & 7.8 &13.5 & 14.1\\
        w. ATP  &2.0$\times$9 &2.2$\times$9 &3.8$\times$7 &7.8$\times$5 & 13.5$\times$3& 14.1$\times$3\\
        \midrule
        SparseGPT& 6.6 & 7.3 &10.0 & 35.0 &60.0 &67.5\\
        w. ATP & 6.6$\times$9 & 7.3$\times$9 &7.3$\times$7 &10.0$\times$5  &60.0$\times$3 &67.5$\times$3\\
        \bottomrule
    \end{tabular}
    }
\end{table}

\subsection{Analyze the Layer-wise Sparsity Distribution.}\label{sec:AnalyzeDistribution}
We analyze the sparsity rate distribution across different average sparsity levels in Figure \ref{fig:comparison_average_rate}. Our findings indicate that at lower average sparsity rates, the differences in sparsity rates across layers are minimal. In contrast, these differences become more pronounced at higher average sparsity rates.
\begin{figure}[htbp]
    \centering
    \includegraphics[width=0.5\linewidth]{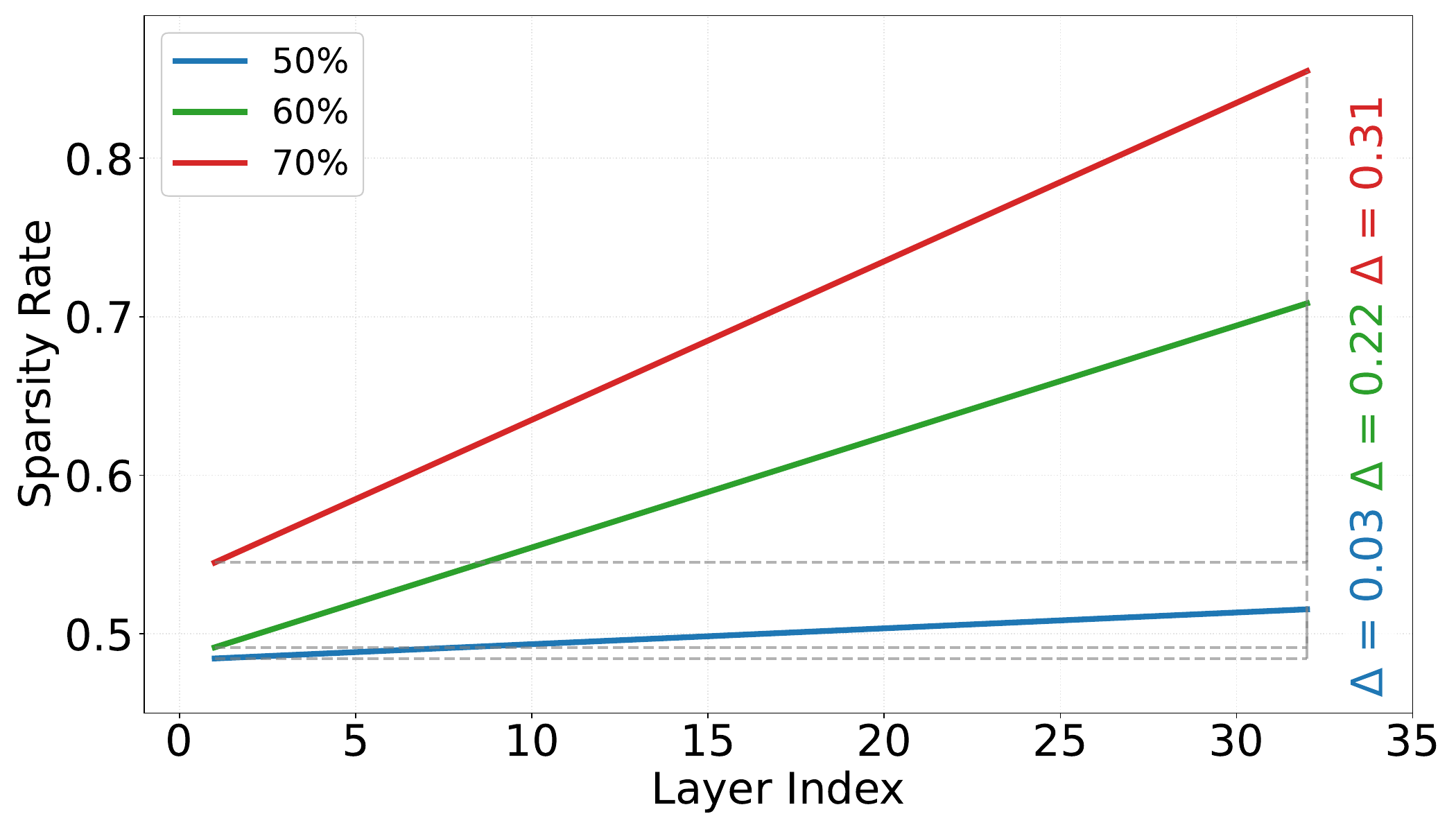}
    \caption{Comparison of layer-wise sparsity rate distributions at different average sparsity levels.}
    \label{fig:comparison_average_rate}
\end{figure}

\subsection{Comparison of Sparsity Rates Distribution with other Methods}\label{ComparisonDistribution}
Figure \ref{fig:comparison_methods} illustrates the sparsity rate distributions obtained from various layer-wise sparsity methods. We observe that the sparsity rates generally follow an increasing pattern from low to high, further validating the rationale behind our method.
\begin{figure}[h]
    \centering
    \includegraphics[width=0.5\linewidth]{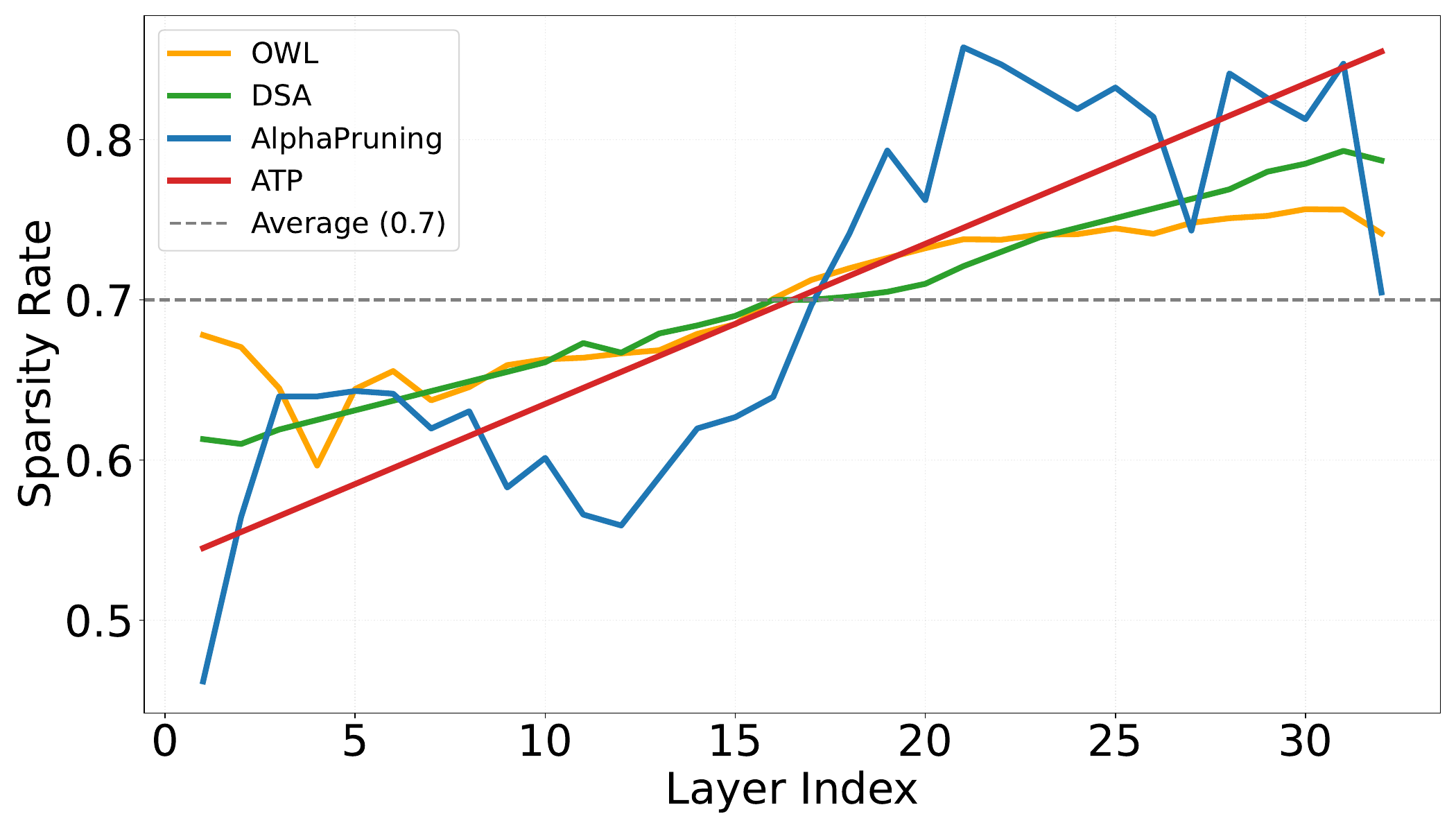}
    \caption{Comparison of layer-wise sparsity rate distribution with other methods.}
    \label{fig:comparison_methods}
\end{figure}

\subsection{Different $\beta$ Settings.}\label{sec:Different_beta}
Figure \ref{fig:different_beta} shows the impact of different \(\beta\) settings on the perplexity of the 70$\%$ sparse LLaMA2-7B model obtained using the Wanda method. We observe that as \(\beta\) increases, the perplexity initially decreases and then rises. Furthermore, the model's perplexity under various \(\beta\) settings remains lower than that of the uniform sparsity rate scheme.
\begin{figure}[h]
    \centering
    \includegraphics[width=0.5\linewidth]{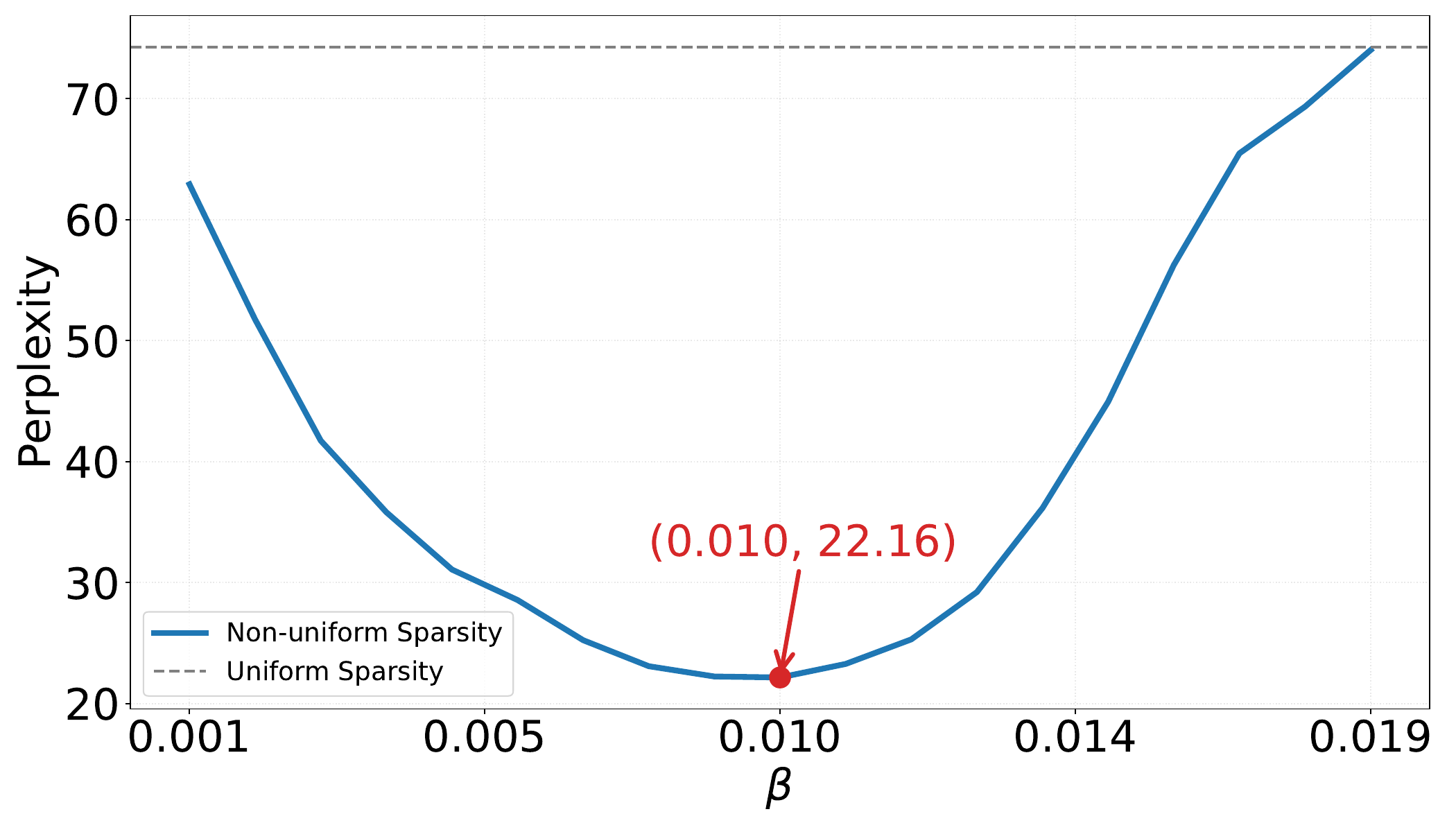}
    \caption{Impact of different \(\beta\) settings on the perplexity of the 70$\%$ sparse LLaMA2-7B.}
    \label{fig:different_beta}
\end{figure}

\section{Detailed Results for Zero-shot Tasks}\label{app:DetailedZero-shotTaskResults}
In this section, we provide a detailed presentation of the performance of each zero-shot task introduced in Sec. \ref{sec:experiments}.
\begin{table*}[h]
  \centering
  \small
  \caption{Zero-shot accuracy results of LLaMA-V1 at 70$\%$ sparsity.
  }\label{tab:zero_shot_llama1}
  \setlength{\tabcolsep}{5.5pt}
  \resizebox{1.0\textwidth}{!}{
  \begin{tabular}{clcccccccc}
  \toprule
Model  &  Method & HellaSwag \hspace{-0.2cm} & Winogrande \hspace{-0.2cm} & BoolQ & OBQA & PIQA & ARC-e & ARC-c & \hspace{-0.2cm}  Mean \\
\midrule 
   \multirow{8}{*}{LLaMA-7B}   & Dense  & 56.92 &  69.93 & 75.05 & 34.40   & 78.67 &75.34 &41.89 & 61.74 \\
  \cmidrule{2-10}
  & SparseGPT & 34.58  &  56.43 & 64.80&16.80 & 64.25 &45.24 &23.12& 43.60\\ 
  &  w. OWL &37.08 &62.35 &66.15 &19.80 & 66.16&48.40 & 26.02& 46.57 \\
   &  w. AlphaPruning & 37.81& 64.01&\bf 68.26 &18.80 &63.60 &46.17& 27.39 &  46.58\\
  \gr \wc  &  \bf w. ATP& \bf37.88  & \bf64.22 & 66.51  & \bf 20.00 & \bf66.17  & \bf49.37  &   \bf27.47  &\bf 47.37 \\ 
   \cmidrule{2-10}
 &   Wanda &28.86  &52.80  &59.69 &14.20 &  57.56 &31.27 &17.75 & 37.45\\ 
   &  w. OWL & 34.89& 58.64& 62.63& 17.60& 64.30&46.97 &24.32 & 44.19 \\
  &  w. DSA &34.68 &59.67 &62.69 &15.40 & 64.09&45.16 & 24.40&43.72  \\
   &  w. AlphaPruning & 36.26&\bf 62.35 &\bf 66.12& 17.20&63.93 &43.90 &25.17 & 44.99 \\
  \gr \wc  &  \bf w. ATP& \bf 37.44 & 61.43 & \bf 65.79 & \bf 20.80 & \bf 67.46 & \bf 50.63 &   \bf 25.68 &\bf 47.03 \\ 
  \midrule
  \multirow{8}{*}{LLaMA-13B}   & Dense &59.94  & 72.77  &77.89 & 33.20   & 79.16 & 77.40 &46.50 & 63.84 \\
  \cmidrule{2-10}
  & SparseGPT & 37.51  & 63.30 &68.78 &20.80 & 67.63 &52.78 &25.17 &48.00\\ 
    &  w. OWL & 40.36&66.22 &66.82 &20.80 &69.86 &57.37 &28.67 & 50.01 \\
   &  w. AlphaPruning &42.43 & 67.80& 67.58&22.00 & 70.01& 55.47&29.10 & 50.63 \\
  \gr \wc  &  \bf w. ATP& \bf 43.31 & \bf67.92 & \bf 69.14 & \bf23.40  & \bf 70.89 & \bf59.97  &   \bf 30.20 &\bf 52.12 \\ 
   \cmidrule{2-10}
 &   Wanda & 31.06 & 54.38 &61.59 &16.20 & 62.68  &42.05 & 17.58& 40.79\\ 
   &  w. OWL &38.57 &63.46 &62.81 &20.40 & 68.61& 57.07&26.37 & 48.18 \\
  &  w. DSA &38.84 &62.04 & 63.03& 23.20&68.98 &58.29 &26.10 & 48.64 \\
   &  w. AlphaPruning &41.04 &64.96 & 64.83& 21.20& 69.01& 59.39&28.24 & 49.81 \\
  \gr \wc  &  \bf w. ATP& \bf43.11  & \bf64.98 & \bf66.64  & \bf25.80  & \bf 69.80 & \bf 59.97 &   \bf30.89  &\bf 51.60 \\ 
  \midrule   
  \multirow{8}{*}{LLaMA-30B}   & Dense & 63.35 &75.69 & 82.69 & 36.00 & 81.01 &80.30 & 52.82 &67.41 \\
  \cmidrule{2-10}
  & SparseGPT  & 44.56 & 69.30 &65.35 &25.80 & 72.42& 65.78&32.25&53.64 \\
   &  w. OWL &46.96 & 72.20&67.95 & 26.80&73.56 & 67.13&35.49 & 55.73 \\
   &  w. AlphaPruning &47.49 &72.30 & 68.96& 30.00&73.50 &67.97 &34.73 & 56.42 \\
  \gr \wc  &  \bf w. ATP& \bf 48.77 & \bf 72.83& \bf 69.28 & \bf30.60  & \bf  73.94& \bf68.81  &   \bf 36.68 &\bf 57.55 \\ 
      \cmidrule{2-10}
  &  Wanda   & 44.23 &67.01 & 66.70 &26.40 & 72.03 & 64.86&32.25  & 53.35 \\
   &  w. OWL & 47.69&69.77 &65.02 & 29.20&73.88 & 68.98& 36.62& 55.88 \\
  &  w. DSA &45.62 & 67.25& 73.15& 27.80& 72.36&68.13 & 32.17& 55.21 \\
   &  w. AlphaPruning & 49.40& 71.27& 63.82&31.00 &73.50  &69.40 &38.31 & 56.67 \\
  \gr \wc  &  \bf w. ATP& \bf50.27  & \bf71.37 & \bf 65.35 & \bf31.40  & \bf 75.41 & \bf 70.62 &   \bf39.42  &\bf57.69  \\ 
  \midrule  
    \multirow{8}{*}{LLaMA-65B}   & Dense  & 64.53&77.27 & 84.89&38.00 & 81.23&81.36 &52.73 &68.57 \\
  \cmidrule{2-10}
  & SparseGPT   &49.98 &74.74 &81.03 &29.20 & 74.76& 72.05&39.51 & 60.18\\
   &  w. OWL &51.25 &74.98 & 81.30& 27.00& 75.68& 68.35&37.71 & 59.47 \\
   &  w. AlphaPruning &52.40 & 74.59& 83.80& 28.42&75.68 & 69.30&38.05 & 60.32 \\
  \gr \wc  &  \bf w. ATP& \bf 53.10 & \bf 75.69& \bf 83.73 & \bf 28.45 & \bf 76.06 & \bf 71.46 &   \bf40.69  &\bf61.31  \\ 
      \cmidrule{2-10}
  &  Wanda   &46.29 &70.79 &77.37 & 27.20&74.21 & 70.66& 37.79&57.76  \\
   &  w. OWL &50.90 & 74.11&78.50 &30.60 &75.68 & 70.70&38.05 & 59.79 \\
  &  w. DSA &48.60 & 73.08& 71.77& 28.80& 75.35& 71.88&38.22 &58.24  \\
   &  w. AlphaPruning &52.55 & 75.06&81.85 &30.40&75.84  & 71.60&40.02 &  61.05\\
  \gr \wc  &  \bf w. ATP& \bf 52.92 & \bf75.70 & \bf 82.73 & \bf31.80  & \bf 75.98 & \bf 73.53 &   \bf  40.78&\bf61.92  \\ 
\bottomrule
\end{tabular}
}
\end{table*}

\begin{table*}[h]
  \centering
  \small
  \caption{Zero-shot accuracy results of LLaMA-V2/V3 at 70$\%$ sparsity.
  }\label{tab:zero_shot_llama23}
  \setlength{\tabcolsep}{5.5pt}
  \resizebox{1.0\textwidth}{!}{
  \begin{tabular}{clcccccccc}
  \toprule
Model  &  Method & HellaSwag \hspace{-0.2cm} & Winogrande \hspace{-0.2cm} & BoolQ & OBQA & PIQA & ARC-e & ARC-c & \hspace{-0.2cm}  Mean \\
\midrule 
  \multirow{8}{*}{LLaMA2-7B}   & Dense &57.17 & 68.90&	77.74&31.40 &78.07	&76.39 & 43.52&61.88\\
  \cmidrule{2-10}
  & SparseGPT   &33.08 &  58.41&  64.89&  17.40 & 62.46 &43.22  &22.01  &43.07\\
     &  w. OWL &36.57 &63.06 &66.94 &21.60 & 63.89&49.33 &24.49 & 46.55 \\
   &  w. AlphaPruning &36.83 &62.19 & 65.93&19.80 & 64.58&49.62 &24.43 & 46.20 \\
  \gr \wc  &  \bf w. ATP& \bf 38.94 & \bf 63.54& \bf69.54  & \bf 21.70 & \bf  68.12& \bf 52.74 &   \bf 25.85 &\bf 48.63 \\ 
      \cmidrule{2-10}
  &  Wanda  & 27.92  & 49.33 & 52.87   & 12.60 &55.33  & 30.60 & 18.69 &35.33\\
    &  w. OWL &31.83 &55.96 & 62.11&16.80 & 61.70& 43.52&20.31 & 41.75 \\
  &  w. DSA&28.54 &50.36 &57.46 &12.00 &57.13 &32.95 & 17.41& 36.55  \\
   &  w. AlphaPruning & 34.56&60.85 & 62.23&18.00 &62.40 &43.27 &22.27 & 43.37 \\
  \gr \wc  &  \bf w. ATP& \bf 36.08 & \bf61.01 & \bf62.39  & \bf 20.40 & \bf 66.81 & \bf 50.76 &   \bf 23.38 &\bf45.83  \\ 
  \midrule 
  \multirow{8}{*}{LLaMA2-13B} & Dense & 60.06 &72.22&80.52 & 35.20 & 79.11 &79.42 & 48.46 &65.00 \\
  \cmidrule{2-10}
  & SparseGPT  & 36.90 & 61.64 &66.02 & 21.00 &67.57 &52.61 &25.94 &47.38\\
     &  w. OWL &39.31 &65.75 &68.04 &22.80 & 67.89 &57.70 &27.82 & 49.90 \\
   &  w. AlphaPruning &41.26 &68.03 & 68.13&24.00 &68.28 &57.15 & 29.18&  50.86\\
  \gr \wc  &  \bf w. ATP& \bf 42.81 & \bf68.09 & \bf 72.91 & \bf24.40  & \bf  69.74& \bf 58.25 &   \bf31.06  &\bf52.46  \\ 
      \cmidrule{2-10}
  &  Wanda   & 29.60 & 51.70 &62.32&13.60 &58.65  & 37.21& 19.11 &  38.88\\
     &  w. OWL &36.31 &60.46 &63.46 &21.80 &67.77 & 55.64&24.91 & 47.19 \\
  &  w. DSA &32.83 & 55.01& 62.41&17.80 & 63.71& 49.87&21.92 &43.36  \\
   &  w. AlphaPruning & 40.28&67.32 & 62.57& 21.60&68.17 &54.46 &29.35& 49.11 \\
  \gr \wc  &  \bf w. ATP& \bf41.44  & \bf67.50 & \bf  74.71& \bf 24.60 & \bf 68.25 & \bf 57.49 &   \bf30.80  &\bf  52.11 \\ 
  \midrule   
   \multirow{8}{*}{LLaMA2-70B}   & Dense  & 66.10&78.06 &83.40	&37.20 &82.21&	82.55	&	54.44& 69.14 \\
  \cmidrule{2-10}
  & SparseGPT   &50.98 &\bf 75.45  & 80.06 &30.00  & 75.24 & 73.57 &  40.61& 60.84\\
    &  w. OWL &51.95 &74.98 &79.25 & 30.40&75.68 & 73.00&40.53 & 60.83 \\
   &  w. AlphaPruning &51.90 &75.06 & 80.40& 29.20&75.68 &74.10 & 40.87& 61.03 \\
  \gr \wc  &  \bf w. ATP& \bf  53.44& \bf76.56 & \bf 80.42 & \bf31.00  & \bf76.71  & \bf75.38  &   \bf 42.23 &\bf62.25  \\ 
      \cmidrule{2-10}
  &  Wanda  & 48.16  & 73.88 & 74.46  & 27.00 &74.86  & 72.69 &38.31 &58.48 \\
    &  w. OWL & 50.20& 74.02&75.01 &28.60 &75.68 &73.02 &38.30 &59.26  \\
  &  w. DSA &48.60 &73.08 &71.77 & 28.80& 75.35&71.88 & 38.22& 58.25 \\
   &  w. AlphaPruning & 51.40& 74.74&73.65 &29.60 & 75.84&72.70 &38.13 & 59.44 \\
  \gr \wc  &  \bf w. ATP& \bf 51.84 & \bf75.53 & \bf78.65  & \bf 29.70 & \bf76.66 & \bf74.83  &   \bf39.16  &\bf 60.91 \\  
   \midrule   
   \multirow{8}{*}{LLaMA3-8B}   & Dense 	& 60.19 & 72.77 & 81.35 & 34.80 &79.71 & 80.09 & 50.43 & 65.62\\
  \cmidrule{2-10}
  & SparseGPT   & 34.26 & 56.75 & 66.51& 16.80 & 63.28 &42.09&21.42 & 43.02\\
     &  w. OWL &36.78 & 58.96&69.54 & 18.20& 65.45& 49.46&24.06 & 46.06 \\
   &  w. AlphaPruning &35.54 &61.56 &71.02 &17.00 & 63.92&46.17 &21.67 & 45.27 \\
  \gr \wc  &  \bf w. ATP& \bf38.19  & \bf63.22 & \bf71.12  & \bf19.00  & \bf66.59  & \bf  50.20&   \bf26.19  &\bf 47.79 \\ 
      \cmidrule{2-10}
 &  Wanda  & 27.36  & 49.96 & 53.33 & 12.00 & 56.04 & 31.86 & 17.41& 35.42 \\
    &  w. OWL & 28.43& 50.43&61.74 &13.00 &57.99 &35.82 &17.58 & 37.85 \\
  &  w. DSA &27.51 & 48.46&54.16 &11.80 & 56.63& 33.03&17.66 & 35.61 \\
   &  w. AlphaPruning &27.82 & 51.85& 56.42&13.40 & 56.47& 34.97& 17.32&36.89 \\
  \gr \wc  &  \bf w. ATP& \bf 31.46 & \bf54.93 & \bf61.79  & \bf16.40  & \bf 62.18 & \bf41.79  &   \bf20.39  &\bf41.28  \\ 
\bottomrule
\end{tabular}
}
\end{table*}

\begin{table*}[h]
  \centering
  \small
  \caption{Zero-shot accuracy results of LLaMA-V1/V2 at 50$\%$ sparsity.
  }\label{tab:zero_shot_50}
  \setlength{\tabcolsep}{5.5pt}
  \resizebox{1.0\textwidth}{!}{
  \begin{tabular}{clcccccccc}
  \toprule
Model  &  Method & HellaSwag \hspace{-0.2cm} & Winogrande \hspace{-0.2cm} & BoolQ & OBQA & PIQA & ARC-e & ARC-c & \hspace{-0.2cm}  Mean \\
\midrule 
  \multirow{7}{*}{LLaMA-7B}   & Dense &76.18 & 70.09&75.11&44.40 &79.16 & 72.98 & 44.71 &66.09 \\
  \cmidrule{2-10}
  &  Wanda  &68.92 &66.38 &70.70&39.00 &74.76 &61.74  &38.91  &60.06 \\
    & w. OWL & 70.06&66.85 &71.44&39.80 &75.52 &69.44 & 40.36 & 61.92\\
  &  w. DSA & 69.51&67.95 &71.31&39.40 &75.46 &69.65 & 40.01 & 61.90 \\
  &  w. ALS & 69.59& 66.30&73.70 & 38.60&\bf77.26 &65.66 &40.02 &61.59   \\
   &  w. AlphaPruning  & 69.60&67.64 &73.33&39.20 & 75.57& 69.65& 38.39 &61.91 \\
  \gr \wc  &  \bf w. ATP& \bf70.18  & \bf68.06 & \bf73.92 &\bf 39.90  & 76.51 & \bf69.74 &   \bf 40.70 &\bf62.72  \\ 
  \midrule 
  \multirow{7}{*}{LLaMA-13B}   & Dense & 79.06& 72.77&77.89&44.80 &80.14 &74.79  & 47.78 & 68.18\\
  \cmidrule{2-10}
  &  Wanda  &74.13 &71.51 &75.96&43.60&77.91 &69.65&43.77&65.22\\
    & w. OWL & 74.82&72.53 &76.39&43.60 &77.37 & 73.32& 43.34 &65.91 \\
  &  w. DSA & 74.02& 70.79&76.05&43.20 &76.98 &72.60 & 44.19 & 65.40 \\
  &  w. ALS & 74.34&71.35 &75.17&43.00 & 77.37& 69.70 &44.45  &65.05 \\
   &  w. AlphaPruning  &74.58 & 72.77&76.63& 44.00& 76.82&74.20 & 44.36 &66.19 \\
  \gr \wc  &  \bf w. ATP& \bf 74.86 & \bf72.78 & \bf76.47 & \bf 44.20& \bf 77.46& \bf74.36 &   \bf 44.60 &\bf 66.39 \\ 
  \midrule 
  \multirow{7}{*}{LLaMA2-7B}   & Dense & 75.98&69.06 &77.74 & 44.20&	79.11& 74.49&46.25 & 66.69\\
  \cmidrule{2-10}
  &  Wanda &68.76 & 67.32& 75.78&41.40 &	76.99&69.23 &41.72 & 63.03 \\
    & w. OWL &70.79 &67.32&75.96&42.80 &76.33& 72.01&41.97&63.88 \\
  &  w. DSA&70.90 &66.45 &76.42 &43.00 &76.22& 71.42&42.83 & 63.89  \\
  &  w. ALS& 70.75& 67.80&75.47 &\bf 44.80 &	77.10&69.61 &42.32 &  64.12 \\
   &  w. AlphaPruning & 70.89&67.48 &76.63 &43.00 &76.33	&72.22 &42.83 &64.20   \\
  \gr \wc  &  \bf w. ATP& \bf70.99  & \bf 67.84& \bf76.73 &  43.70 & \bf76.44  & \bf72.90 &   \bf42.85  &\bf64.49  \\ 
  \midrule 
  \multirow{7}{*}{LLaMA2-13B}   & Dense &79.39 &72.38 & 80.58& 45.20&80.52 &77.53 &49.15  &69.25 \\
  \cmidrule{2-10}
  &  Wanda  & 75.02&69.39 &80.34 &44.10 &78.13 & 70.37 &42.76 &65.73 \\
    & w. OWL  & 76.11&71.19 &81.65 &45.40 &78.67 & 76.85&46.24 &68.02 \\
  &  w. DSA &75.86 &71.03 & 80.83& 45.00 &78.62 &76.22 &45.99 &67.65  \\
  &  w. ALS &75.67 &72.06 & 81.35&\bf 45.80 &78.51 &70.33 & 46.08 &67.11\\
   &  w. AlphaPruning  & 76.19& 71.58& 80.97&45.00 &78.34 &76.38 &46.16 &  67.80\\
  \gr \wc  &  \bf w. ATP& \bf76.26  & \bf72.09 & \bf81.66 &   45.20& \bf 78.74& \bf76.86 &   \bf46.42  &\bf68.18  \\ 
\bottomrule
\end{tabular}
}
\end{table*}

\begin{table*}[h]
  \centering
  \small
  \caption{Zero-shot accuracy results of more LLM architectures at 70$\%$ sparsity.
  }\label{tab:zero_shot_more_LLM}
  \setlength{\tabcolsep}{5.5pt}
  \resizebox{1.0\textwidth}{!}{
  \begin{tabular}{clcccccccc}
  \toprule
Model  &  Method & HellaSwag \hspace{-0.2cm} & Winogrande \hspace{-0.2cm} & BoolQ & OBQA & PIQA & ARC-e & ARC-c & \hspace{-0.2cm}  Mean \\
\midrule 
   \multirow{3}{*}{LLaMA3.1-8B}   & Dense & 59.98 &73.32  &82.05 &33.20 &79.98& 81.57&51.45 &65.93 \\
  \cmidrule{2-10}
 &  Wanda   &27.43 &48.70 &57.71 & 13.60& 55.01&31.86 &18.43 & 36.10\\
  \gr \wc  &  \bf w. ATP& \bf 31.81 & \bf55.49 & \bf 62.08& \bf15.60& \bf 62.68& \bf  42.63&   \bf 20.82 &\bf 41.59\\ 
  \midrule 
   \multirow{3}{*}{OPT-13B}   & Dense  & 52.43& 65.04&	65.93& 27.20&75.84 &	67.13&32.94 &55.22\\
  \cmidrule{2-10}
 &  Wanda  & 34.36 & 55.09 & 55.02 & 15.60 & 62.89 &43.73 & 23.89 &41.51 \\
  \gr \wc  &  \bf w. ATP& \bf 36.47 & \bf 58.09& \bf62.17 & \bf18.20 & \bf65.56 & \bf 46.63 &   \bf 24.91 &\bf44.58 \\ 
    \midrule 
   \multirow{3}{*}{Vicuna-13B}   & Dense  & 59.64&71.59 &	85.26 &36.80 &79.00 & 78.66&	47.78& 65.53\\
  \cmidrule{2-10}
 &  Wanda  & 31.84 &54.70  & 62.78 & 16.40 & 61.75 & 44.87 &22.10  &42.06\\
  \gr \wc  &  \bf w. ATP& \bf 41.47 & \bf63.85 & \bf76.70 & \bf25.00 & \bf 69.04& \bf 61.95 &   \bf34.30  &\bf53.19 \\ 
      \midrule 
   \multirow{3}{*}{Qwen2.5-7B}   & Dense  &60.01 & 72.85&84.65&33.20 & 78.73&80.43 &	47.70& 65.36\\
  \cmidrule{2-10}
 &  Wanda  &30.68 &51.93 & 61.96&15.20 & 61.59&46.34 &20.05	&41.10 \\
  \gr \wc  &  \bf w. ATP& \bf32.84  & \bf56.59 & \bf62.14 & \bf16.40 & \bf63.44 & \bf 46.42 &   \bf21.93  &\bf 43.82 \\ 
    \midrule 
   \multirow{3}{*}{Mistral-7B}   & Dense  & 61.21& 73.88& 83.64&32.60 & 80.58& 80.85&50.43&66.17 \\
  \cmidrule{2-10}
 &  Wanda  &28.86 & 51.07&60.03 &12.60 &57.56 &34.60 &18.69	&37.62 \\
  \gr \wc  &  \bf w. ATP& \bf34.44  & \bf58.96 & \bf 62.20& \bf15.60 & \bf63.93 & \bf 42.68 &   \bf 21.50 &\bf42.76 \\ 
\bottomrule
\end{tabular}
}
\end{table*}

\begin{table*}[h]
  \centering
  \small
  \caption{Zero-shot accuracy results of more post-training sparsity methods at 70$\%$ sparsity.
  }\label{tab:zero_shot_more_post-training}
  \setlength{\tabcolsep}{5.5pt}
  \resizebox{1.0\textwidth}{!}{
  \begin{tabular}{clcccccccc}
  \toprule
Model  &  Method & HellaSwag \hspace{-0.2cm} & Winogrande \hspace{-0.2cm} & BoolQ & OBQA & PIQA & ARC-e & ARC-c & \hspace{-0.2cm}  Mean \\
\midrule 
   \multirow{3}{*}{LLaMA2-7B} & Dense  &57.17 & 68.90&	77.74&31.40 &78.07	&76.39 & 43.52&61.88 \\
  \cmidrule{2-10}
 &  DSnoT   & 27.80&51.78 &49.66 &12.60 & 55.66&30.89 & 17.41& 35.11\\
\gr \wc  &  \bf w. ATP& \bf 34.98 & \bf58.01 & \bf62.23 & \bf18.20 & \bf65.78 & \bf 50.42 &   \bf 21.93 &\bf 44.51\\ 
    \cmidrule{2-10}
 &  Pruner-Zero&27.56&50.99&41.93&13.00&56.90&34.47&18.60&34.78\\
  \gr \wc  &  \bf w. ATP& \bf 35.97 & \bf57.93 & \bf64.46 & \bf18.60 & \bf 64.15& \bf 48.27 &   \bf23.98  &\bf 44.77\\ 
      \cmidrule{2-10}
 &  ALPS&38.35&61.96&64.59&22.20&66.82&48.37&24.95&46.75\\
 \gr \wc  &  \bf w. ATP& \bf 41.58& \bf 64.25 & \bf64.78 & \bf 24.20 & \bf68.34 & \bf 56.65 &   \bf 28.92 &\bf 49.82\\ 
\bottomrule
\end{tabular}
}
\end{table*}